\pdfoutput=1
\documentclass[a4paper,10pt]{article}
\usepackage[utf8]{inputenc}
\usepackage[margin=1in]{geometry}

\usepackage{amsmath,amsfonts,amsthm,amssymb}
\usepackage{mathtools}
\usepackage[round]{natbib}
\usepackage{enumitem}
\usepackage{xspace}

\usepackage{bm}
\usepackage[extdef=true]{delimset}
\usepackage[colorlinks,allcolors=blue]{hyperref}
\usepackage[round]{natbib}
\usepackage{xspace}
\usepackage{enumitem}
\usepackage[algo2e,ruled]{algorithm2e}
\SetKwComment{Comment}{$\triangleright$\ }{}
\usepackage[capitalize]{cleveref}
\usepackage{booktabs}
\usepackage{adjustbox}
\usepackage{tablefootnote}
\usepackage{threeparttable}
\usepackage{caption}
\usepackage{multirow}
\usepackage{tcolorbox}
\usepackage{comment} %

\usepackage{thmtools,thm-restate}
\declaretheorem[parent=]{theorem}
\declaretheorem[parent=]{lemma}

\DeclareMathOperator*{\argmax}{arg\,max}
\DeclareMathOperator*{\argmin}{arg\,min}

\DeclareMathOperator*{\polylog}{poly\,log}

\newcommand{\halfcmark}{\checkmark\kern-5pt\raisebox{2pt}{\scriptsize\textbackslash}\kern1pt}
\newcommand{\ifrac}[2]{{#1}/{#2}}

\newcommand{\reals}{{\mathbb{R}}}

\newcommand{\E}{\mathbb{E}}

\newcommand{\event}{E}
\newcommand{\dotp}{\boldsymbol{\cdot}}
\newcommand{\eqdef}{\triangleq}

\newcommand{\diam}{D}
\newcommand{\lip}{G}
\newcommand{\sm}{\beta}

\newcommand{\initg}{\gamma}
\newcommand{\initeta}{\alpha}

\newcommand{\alg}{\mathcal{A}}
\newcommand{\cF}{\mathcal{F}}

\newcommand{\domain}{\Omega}
\newcommand{\proj}[1]{\Pi_\domain\brk*{#1}}
\newcommand{\projop}{\operatorname{Proj}}

\newcommand{\oracle}{\mathcal{O}}

\newcommand{\adasgd}[0]{\textsf{AdaSGD}\xspace}
\newcommand{\adapsgd}[0]{\textsf{AdaPSGD}\xspace}
\newcommand{\sgd}[0]{\textsf{SGD}\xspace}

\newcommand{\Otilde}{\smash{\widetilde{O}}}
\newcommand{\Ohat}{\smash{\widehat{O}}}

\newcommand{\diamstar}{\diam_\star}
\newcommand{\diammin}{\diam_{\min}}
\newcommand{\diammax}{\diam_{\max}}
\newcommand{\lipstar}{\lip_{\star}}
\newcommand{\lipmin}{\lip_{\min}}
\newcommand{\lipmax}{\lip_{\max}}
\newcommand{\lipball}{\lip_{8\star}}
\newcommand{\sigmastar}{\sigma_{\star}}
\newcommand{\sigmamin}{\sigma_{\min}}
\newcommand{\sigmamax}{\sigma_{\max}}
\newcommand{\sigmaball}{\sigma_{8\star}}
\newcommand{\smstar}{\sm_{\star}}
\newcommand{\smmin}{\sm_{\min}}
\newcommand{\smmax}{\sm_{\max}}

\newcommand{\fstar}{F_{\star}}
\newcommand{\fmin}{F_{\min}}
\newcommand{\fmax}{F_{\max}}
\newcommand{\etamin}{\eta_{\min}}
\newcommand{\etamax}{\eta_{\max}}

\newcommand{\wout}{\overline{w}}
\newcommand{\foracle}{\widetilde{f}}
\newcommand{\goracle}{\widetilde{g}}
\newcommand{\fsigma}{\sigma_{0}}

\title{How Free is Parameter-Free Stochastic Optimization?}
\author{%
    Amit Attia%
    \thanks{\scriptsize Blavatnik School of Computer Science, Tel Aviv University; \texttt{amitattia@mail.tau.ac.il}.}
    \and
    Tomer Koren%
    \thanks{\scriptsize Blavatnik School of Computer Science, Tel Aviv University, and Google Research Tel Aviv; \texttt{tkoren@tauex.tau.ac.il}.}
}

\begin{document}
\maketitle

\begin{abstract}
    We study the problem of parameter-free stochastic optimization, inquiring whether, and under what conditions, do fully parameter-free methods exist: these are methods that achieve convergence rates competitive with optimally tuned methods, without requiring significant knowledge of the true problem parameters. Existing parameter-free methods can only be considered ``partially'' parameter-free, as they require some non-trivial knowledge of the true problem parameters, such as a bound on the stochastic gradient norms, a bound on the distance to a minimizer, etc. In the non-convex setting, we demonstrate that a simple hyperparameter search technique results in a fully parameter-free method that outperforms more sophisticated state-of-the-art algorithms. We also provide a similar result in the convex setting with access to noisy function values under mild noise assumptions. Finally, assuming only access to stochastic gradients, we establish a lower bound that renders fully parameter-free stochastic convex optimization infeasible, and provide a method which is (partially) parameter-free up to the limit indicated by our lower bound.
\end{abstract}

\section{Introduction}

Stochastic first-order optimization is a cornerstone of modern machine learning, with stochastic gradient descent~\citep[SGD,][]{robbins1951stochastic} as the go-to method for addressing statistical learning problems.
The tuning of algorithmic parameters, particularly those associated with SGD such as the step-size, proves to be a challenging task, both theoretically and in practice~\citep[e.g.,][]{bottou2012stochastic,schaul2013no}, especially when the problem parameters are unknown. 
The present paper focuses on theoretical aspects of this challenge.

Addressing the challenge, a variety of methods emerged over the years, with the goal of minimizing knowledge about problem parameters required for tuning, while maintaining performance competitive with carefully tuned algorithms.
These include the so-called
adaptive methods, such as AdaGrad and Adam~\citep[e.g.,][]{duchi2011adaptive,kingma2014adam} 
and their recent theoretical advancements
\citep{reddi2018convergence,tran2019convergence,kavis2019unixgrad,alacaoglu2020new,Faw2022ThePO,kavis2022high,attia2023sgd,liu2023high}, that self-tune learning rates based on gradient statistics;
parameter-free methods~\citep[e.g.,][]{chaudhuri2009parameter,Streeter2012NoRegretAF,luo2015achieving,orabona2021parameter,carmon2022making} that focus primarily on automatically adapting to the complexity of (i.e., distance to) an optimal solution;
and advanced techniques from the literature on online learning \citep{orabona2016coin,cutkosky2018black} that can be used in a stochastic optimization setting via an online-to-batch reduction. 
More recently, several works have achieved advancements in improving practical performance, narrowing the gap to finely-tuned methods \citep{orabona2017training,chen2022better,Ivgi2023DoGIS,pmlr-v202-defazio23a,mishchenko2023prodigy}.

While adaptive and parameter-free methods enjoy strong theoretical guarantees, each of the aforementioned approaches still demands some level of non-trivial knowledge regarding problem parameters. 
This includes knowledge about the smoothness parameter and the initial suboptimality for non-convex adaptive and self-tuning methods, or a bound on either the stochastic gradient norms or on the distance to an optimal solution for parameter-free convex (including online) optimization algorithms.
To the best of our knowledge, in the absence of such non-trivial assumptions, no method in the literature is fully parameter-free, in either the non-convex or convex, smooth or non-smooth cases.

This paper explores the following question: in what stochastic optimization scenarios, and under what conditions, do fully parameter-free optimization algorithms exist?  
In this context, we note that direct hyperparameter tuning, which parameter-free self-tuning methods aim to eschew, could be considered a valid approach for designing parameter-free algorithms. 
Consequently, we also examine the question: can self-tuning optimization algorithms actually outperform direct tuning methods based on simple hyperparameter search?

We begin our inquiry in the general setting of non-convex (smooth) optimization.
We observe that a simple hyperparameter search technique for tuning of fixed stepsize SGD results in a fully parameter-free method with convergence rate that matches \emph{optimally tuned} SGD up to poly-logarithmic factors.
While extremely straightforward, this approach outperforms the state-of-the-art bounds for adaptive methods in this context~\citep{Faw2022ThePO,kavis2022high,attia2023sgd,liu2023high} that require non-trivial knowledge on the smoothness parameter and the initial sub-optimality for optimal tuning.

Moving on to the convex optimization setting, we observe that a similar hyperparameter search technique can be used to design a simple and efficient parameter-free algorithm, provided access to noisy function value queries to the objective. Under very mild assumptions on the scale of noise in the value queries, the obtained algorithm is fully parameter-free with convergence rate matching that of perfectly tuned SGD (up to poly-log factors), in both the smooth and non-smooth convex cases.
This algorithm thus outperforms existing results in the literature, albeit under an additional assumption of a reasonably-behaved stochastic function value oracle. Indeed, an inspection of previous studies in this context reveals that they neglect access to function values and only rely on stochastic gradient queries (even though function value access is often readily available in practical applications of these methods).

Our last area of focus is therefore the convex optimization setting with access solely to stochastic gradients (and no access whatsoever to function values), which emerges as the only general stochastic optimization setting where achieving parameter-freeness is potentially non-trivial.
Indeed, and perhaps somewhat surprisingly, we identify in this setting a limitation to obtaining a fully parameter-free algorithm: we establish a lower bound showing that any optimization method cannot be parameter-free unless it is provided with either the gradient-noise magnitude or the distance to a minimizer, each up to a (multiplicative) factor of $O(\sqrt{T})$ where $T$ is the number of optimization steps.
The bound sheds light on why prior methods require non-trivial knowledge of problem parameters, yet still leaves room for improvement: state-of-the-art results in this context \citep[e.g.,][]{carmon2022making,Ivgi2023DoGIS} assume knowledge on the magnitude of the stochastic gradients, whereas our lower bound only necessitates knowledge on the magnitude of the \emph{noise component} in the gradient estimates. 

Our final contribution complements the lower bound and introduces a parameter-free method in the convex non-smooth (Lipschitz) setting, achieving the same rate as of optimally-tuned SGD, up to poly-log factors and an unavoidable term prescribed by our lower bound. When given a bound on the noise in the stochastic gradients accurate up to a $O(\sqrt{T})$ factor, our method becomes fully parameter-free.
The same method achieves a similar parameter-free guarantee in the setting of convex smooth optimization, nearly matching the rate of tuned SGD in this case.

\subsection{Summary of contributions}

To summarize our results in some more detail, let us first specify the setup of parameter-free optimization slightly more concretely.
In all cases, we consider unconstrained optimization of an objective function $f : \reals^d \to \reals$.
Rather than directly receiving the ground-true problem parameters (e.g., smoothness parameter, Lipschitz constant, etc.), %
the algorithms are provided with a range (lower and upper bounds) containing each parameter. 
We refer to a method as parameter-free if its convergence rate matches a benchmark rate (e.g., that of optimally-tuned SGD), up to poly-logarithmic factors in the ranges parameters (as well as in the number of steps $T$ and the probability margin $\delta$).

    \paragraph{Non-convex setting: fully parameter-free algorithm.}
    We give a fully parameter-free method with the same convergence rate as \emph{tuned} SGD up to poly-log factors.
    Given an initialization $w_1 \in \reals^d$, number of queries $T$, probability margin $\delta$ and a range $[\etamin,\etamax]$ containing the theoretically-tuned SGD stepsize (the latter can be computed from the problem parameter ranges),
    the algorithm produces $\wout \in \reals^d$ such that with probability at least $1-\delta$,
    \begin{align*}
    \norm{\nabla f(\wout)}^2 
    &\leq
    \brk*{
        \frac{\smstar \fstar
        + \sigmastar^2}{T}
        \!+\! \frac{\sigmastar \sqrt{\smstar \fstar}}{\sqrt{T}}
    }\polylog\brk*{\tfrac{\etamax}{\etamin},\tfrac{1}{\delta}}
    ,
    \end{align*}
    where $\smstar$ is the smoothness parameter, $\fstar=f(w_1)-\min f(w)$, and $\sigmastar$ is a noise bound of the stochastic gradients.
    While the algorithm is based on simple observations, it is fully parameter-free and provides stronger guarantees with a simpler analysis compared to existing adaptive or self-tuning methods in this setting.

    \paragraph{Convex setting: fully parameter-free algorithm with noisy function values.}

    Moving to the convex case, we first consider the canonical setting with access to noisy function values. 
    We devise a simple parameter-free method, that given $w_1$, $T$, $\delta$ and a range $[\etamin,\etamax]$ containing the theoretically-tuned SGD stepsize, produces $\wout \in \reals^d$ such that with probability $\geq 1-\delta$,
    \begin{align*}
        f\brk*{\wout}-f(w^\star)
        & \! \leq \!
        \brk4{\frac{\diamstar \sqrt{\lipstar^2+\sigmastar^2} }{\sqrt{T}} \!+\! \frac{\fsigma}{\sqrt{T}} }
        \polylog\brk*{\tfrac{\etamax}{\etamin},\tfrac{1}{\delta}}
        ,
    \end{align*}
    where $\lipstar$ is the Lipschitz constant, $\diamstar=\norm{w_1-w^\star}$ ($w^\star$ is a minimizer of $f$), and $\fsigma$ and $\sigmastar$ are noise bounds of the function and gradient oracles respectively.
    Under very mild assumptions on the function values noise $\fsigma$, the method is fully parameter-free and achieves the same rate as \emph{tuned} SGD up to poly-log factors.%
    
    \paragraph{Convex setting: impossibility without function values.}

    Next we consider the convex setting without any access to function values.
    We establish that without further assumptions, no convex optimization method can be fully parameter-free while nearly matching the rate of \emph{tuned} SGD.
    We show that for any $\alpha \in [1,\frac34\sqrt{T}]$ and every $T$-queries deterministic algorithm receiving ranges $[\tfrac{1}{a\sqrt{T}}\diammax,\diammax]$, $[\tfrac{1}{a \sqrt{T}}\sigmamax,\sigmamax]$ and known $\lipstar=\tfrac{\sigmamax}{(2T-1)}$, there exist $\diamstar$ and $\sigmastar$ that belong to the ranges (resp.), a convex and $\lipstar$-Lipschitz function $f$ with a minimizer $w^\star$ such that $\norm{w^\star} = \diamstar$, and a $\sigmastar$-bounded first-order oracle such that with constant probability, the output of the algorithm, $\wout$, satisfy
    \begin{align*}
        f(\wout)-f(w^\star)
        &\geq
        \frac{\diamstar (\lipstar+\sigmastar)\alpha}{6\sqrt{T}}
        \geq
        \frac{\diamstar \sigmamax}{6 T}
        .
    \end{align*}
    Thus, without non-trivial prior knowledge on the parameter ranges, no algorithm can match the performance of \emph{tuned} SGD and must include a term linear in $\sigmamax$.

    \paragraph{Parameter-free algorithm for the convex non-smooth setting.}
    
    Assuming only gradient access, we propose a method which requires knowledge of  $\diammin,\diammax,\lipmax$ and $\sigmamax$ such that $\diammin \leq \diamstar = \norm{w_1-w^\star} \leq \diammax$, $\lipstar \leq \lipmax$ and $\sigmastar \leq \sigmamax$ and produces $\wout$ such that with probability at least $1-\delta$,
    \begin{align*}
        f(\wout)-f^\star
        &\leq
        \brk*{\frac{\diamstar(\lipstar+\sigmastar)}{\sqrt{T}} \!+\! \frac{\diamstar \sigmamax}{T}}
        \! \polylog\brk*{\tfrac{\diammax}{\diammin},\lipmax,\sigmamax,\tfrac{T}{\delta}}
        .
    \end{align*}
    The method is parameter-free if $\sigmamax$ is provided up to a tolerance of $O(\sqrt{T})$, in which case it achieves the same rate of convergence as \emph{tuned} SGD up to logarithmic factors. 
    This is the maximal tolerance possible in light of our lower bound, which shows that the additional $\ifrac{\sigmamax}{T}$ term is unavoidable without further assumptions. 
    
    \paragraph{Parameter-free algorithm for the convex smooth setting.}
    
    Assuming that the objective is $\smstar$-smooth rather than Lipschitz, the same convex optimization method with parameters $\diammin$, $\diammax$, $\smmax$ and $\sigmamax$ such that $\diammin \leq \diamstar \leq \diammax$, $\smstar \leq \smmax$ and $\sigmastar \leq \sigmamax$, produce $\wout \in \reals^d$ such that with probability at least $1-\delta$,
    \begin{align*}
        f(\wout)-&f^\star
        \leq
        \bigg(
        \frac{\smstar \diamstar^2}{T}+\frac{\diamstar \sigmastar}{\sqrt{T}}
        + \frac{\diamstar \sigmamax}{T}
        \bigg)
        \polylog\brk*{\diammax,\tfrac{1}{\diammin},\smmax,\sigmamax,\tfrac{T}{\delta}}
        .
    \end{align*}
    To the best of our knowledge, this is the first parameter-free method for stochastic convex and smooth optimization (in the case where $\sigmastar$ is provided up to $O(\sqrt{T})$ tolerance, which is again required to account for the unavoidable term prescribed by our lower bound).

\subsection{Additional related work}

\paragraph{Adaptive stochastic non-convex optimization.}
A long line of work focus on the analysis of SGD with AdaGrad-type stepsizes \citep{ward2019adagrad,li2019convergence,kavis2022high,Faw2022ThePO,attia2023sgd,liu2023high}, which we refer to as \adasgd (also known as AdaGrad-norm, e.g., \citealt{ward2019adagrad,Faw2022ThePO}).
\adasgd enjoys high-probability guarantees and rate interpolation for small noise \citep{attia2023sgd,liu2023high}.
Compared to our tuning-based fully parameter-free method, the analysis of \adasgd is cumbersome and sub-optimal without knowing the smoothness and function sub-optimality.

\paragraph{Parameter-free and adaptive stochastic convex optimization methods.}
Existing methods assume either known diameter or known stochastic gradient norm bound.
Assuming a known diameter, \citep{kavis2019unixgrad} presented a method with an (almost) optimal rate for convex optimization and an accelerated rate for convex smooth optimization.
It is an open question whether parameter-free acceleration is possible.
Given a a bound of the stochastic gradient norm, \citet{carmon2022making} provided the first method which achieve the optimal rate up to a double-logarithmic factor in the diameter range by utilizing an efficient bisection procedure to find a good SGD stepsize.
\citet{Ivgi2023DoGIS} introduced a parameter-free method which use dynamic stepsizes for SGD, and demonstrated strong practical performance.
Both methods are parameter-free if the bound of the stochastic gradients norm is given up to a tolerance of $O(\sqrt{T})$, while our method and analysis require only such bound for the noise of the stochastic gradients; this distinction is crucial to our guarantee in the convex and smooth case, where the gradients are unbounded.

\paragraph{Parameter-free deterministic optimization methods.}
While our focus here is on parameter-free stochastic optimization, in the context of \emph{deterministic} optimization several fully parameter-free methods were previously suggested.
The adaptive Polyak method of \citet{hazan2019revisiting} achieves almost the same rate as tuned gradient descent for convex Lipschitz and convex smooth problems.
For convex smooth problems,
\citet{beck2009fast} and \citet{nesterov2015universal} use line-search techniques to obtain the optimal (accelerated) rate.
It is also worth mentioning the tuning-free methods of \citet{pmlr-v202-defazio23a,mishchenko2023prodigy}, which are not fully parameter-free (they need knowledge of the Lipschitz constant) yet demonstrate strong practical performance without tuning.

\paragraph{Parameter-free online convex optimization.}
In the setting of online convex optimization (OCO), several works concerned themselves with parameter-free~\citep[e.g.,][]{mcmahan2014unconstrained,orabona2016coin,cutkosky2018black,mhammedi2020lipschitz} and scale-free~\citep[e.g.,][]{orabona2018scale,mhammedi2020lipschitz} methods. The parameter-free OCO literature assumes known gradient norm bound and achieve regret of $\Otilde(\norm{u} \lip_{\max} \sqrt{T \log\norm{u}})$ where $u,\lip_{\max}$ and $T$ are an arbitrary comparator, gradient norm bound and number of rounds respectively \citep{mcmahan2014unconstrained}. \citet{cutkosky2016online,cutkosky2017online} established lower bounds which rule out such regret bounds if both the comparator and gradient norm bounds are unknown.
Sidestepping the lower bound, several works designed parameter-free methods (with unknown Lipschitz bound and comparator norm), suffering an additional $\widetilde O (\norm{u}^3 \lip_{\max})$ term \citep{cutkosky2019artificial,mhammedi2020lipschitz}.
Considering online-to-batch conversions, online parameter-free methods with known gradient norm bound still require a non-trivial bound of the stochastic gradient norms, not only the noise bound. 
\citet{jun2019parameter} relaxed the bounded gradients assumption using stochastic gradients, but requires a known bound of the expected gradient norms instead.
We note that the method of \citet{cutkosky2019artificial} can use crude range bounds for a guarantee of $\widetilde O (\diamstar \lipstar / \sqrt{T} + \diammax \lipstar / T)$, which is parameter-free when $\diammax/\diammin = O(\sqrt{T})$. Diameter tolerance is a promising future research direction, as most current studies focus on gradient norm/noise tolerance.

\paragraph{Classical analyses of stochastic gradient descent.}

\citet{ghadimi2013stochastic} were the first to examine stochastic gradient descent in the smooth non-convex setting and obtain tight convergence. They demonstrated that a \emph{properly tuned} SGD with $T$ steps achieves a rate of $O\brk{\ifrac{1}{T}+\ifrac{\sigma}{\sqrt{T}}}$ under the assumption of uniformly bounded variance $\sigma^2$; \citet{arjevani2022lower} provided a tight lower bound.
Assuming sub-Gaussian noise, \citet{ghadimi2013stochastic} provided a high-probability convergence rate of $\widetilde O\brk{\ifrac{1}{T}+\ifrac{\sigma^2}{T}+\ifrac{\sigma}{\sqrt{T}}}$ by amplifying the success probability using multiple runs of SGD, while \citet{liu2023high} established a similar rate for a single run of SGD.
For the convex setting, \citet{lan2012optimal} provided convergence guarantees for stochastic gradient descent in composite stochastic optimization, establishing a convergence rate of $O\brk{\ifrac{(\lip+\sigma)}{\sqrt{T}}}$ for convex and $\lip$-Lipschitz objectives and $O\brk{\ifrac{1}{T}+\ifrac{\sigma}{\sqrt{T}}}$ for convex smooth objectives, assuming uniformly bounded variance of $\sigma^2$. Classical non-asymptotic analyses of SGD and minimax lower bounds in stochastic optimization trace back to \citet{nemirovskij1983problem}.

\paragraph{Addendum: concurrent work on parameter-free stochastic optimization.}

Shortly after the present manuscript first appeared on arXiv \citep{attia2024free}, other works appeared which explore similar questions in parameter-free stochastic optimization.
\citet{carmon2024price} study the ``price of adaptivity'' to uncertainty in problem parameters in stochastic convex problems, with focus on characterizing the tight penalty (i.e., poly-logarithmic factors in the bound) one must pay for being parameter-free. They provide several lower bounds for different uncertainty scenarios, including stochastic gradients with bounded norms and stochastic gradients with a bounded second moment, as well as a result similar to our \cref{thm:lower-bound-new} that quantifies the price of uncertainty in both diameter and stochastic gradient norms.
In contrast, our primary focus is on studying conditions under which parameter-free algorithms are possible in various optimization scenarios, modulo poly-logarithmic factors in the convergence bounds.

In an independent work, \citet{khaled2024tuning} study parameter-free optimization in the non-convex and convex settings and establish a set of results closely related to ours.
For non-convex smooth optimization, they provide an upper bound similar to our \cref{thm:non-convex} (they also give a lower bound for the in-expectation rate of SGD, showing that it is unattainable by parameter-free methods).
For convex optimization, they prove that full parameter-freeness is in general impossible, while showing a convergence result that inversely depends (polynomially) on a certain SNR parameter, that for favorable noise distributions is bounded away from zero, in which case the algorithm is parameter-free.
Comparing to our lower bound in \cref{thm:lower-bound-new}, we provide a detailed characterization of the cost of being parameter-free in this setting, showing how convergence of parameter-free methods degrades with the degree of uncertainty in the gradient noise parameter.
In terms of upper bounds, we establish feasibility of parameter-free convex optimization with no dependence on SNR (that can be arbitrarily small even for bounded noise distributions), with rates matching our lower bound up to poly-log factors.
In addition, we establish an analogous upper upper bound for convex smooth problems, which is, to our knowledge, the first parameter-free result in this setting.

\section{Preliminaries}
\label{sec:setup-main-results}

\paragraph{Notation.} Throughout, we use $\log$ to denote the base $2$ logarithm and $\log_+(\cdot) \eqdef 1+\log(\cdot)$. 
$\Otilde$ and $\Ohat$ are asymptotic notations that suppress logarithmic and doubly-logarithmic factors respectively (in addition to constant factors).

\subsection{Optimization setup}

Our focus is unconstrained stochastic optimization in $d$-dimensional Euclidean space $\reals^d$ with the $\ell_2$ norm, denoted $\norm{\cdot}$.
We assume a stochastic first-order access to a differentiable objective function $f : \reals^d \to \reals$ through a randomized oracle $\oracle$. 
We will consider two variants of such access, common to the literature on stochastic optimization:
\begin{enumerate}[label=(\roman*)]%
    \item\emph{\bfseries Stochastic first-order oracle:}
    The oracle $\oracle$ generates, given any $w \in \reals^d$, unbiased value and gradient samples for $f$ at the point $w$, namely $\oracle(w) = (\foracle(w),\goracle(w))$ such that
    for all $w \in \reals^d$, $\E[\foracle(w)]=f(w)$ and $\E[\goracle(w)]=\nabla f(w).$
    We will make the standard assumption that $\foracle$ and $\goracle$ have uniformly bounded noise,%
    \footnote{The bounded noise assumption is made here for simplicity and can be relaxed to a sub-Gaussian assumption on the noise; see, for example, \citet{attia2024note} or Appendix~C of \citet{attia2023sgd} for further details.}
    that is, for some parameters $\fsigma > 0$ and $\sigmastar > 0$,
    \begin{align*}
        \forall ~ w \in \reals^d ~:
        \qquad
        &
        \Pr\brk{\abs{\foracle(w)- f(w)} \leq \fsigma} = 1
        \qquad \text{and}\qquad
        \Pr\brk{\norm{\goracle(w)-\nabla f(w)} \leq \sigmastar} = 1
        .
    \end{align*}
    
    \item\emph{\bfseries Stochastic gradient oracle:}
    The oracle $\oracle$ only generates a stochastic gradient sample with $\sigmastar$-bounded noise for $f$ given a point $w$; namely $\oracle(w) = (\goracle(w))$ such that
    \begin{align*}
        & \forall ~ w \in \reals^d ~ : \qquad 
        \E[\goracle(w)]=\nabla f(w)
        \qquad \text{and} \qquad 
        \Pr\brk{\norm{\goracle(w)-\nabla f(w)} \leq \sigmastar} = 1
        .
    \end{align*}

\end{enumerate}

In this framework, we distinguish between several different optimization scenarios:
\begin{enumerate}[label=(\roman*)]%
\item\emph{\bfseries Non-convex setting.}
In this case we will assume the objective is not necessarily convex but is $\smstar$-smooth.%
\footnote{A function $f : \reals^d \to \reals$ is said to be $\sm$-smooth if $\norm{\nabla f(x)-\nabla f(y)} \leq \sm \norm{x-y}$ for all $x,y \in \reals^d$. In particular, this implies that $\abs{f(y) - f(w) - \nabla f(x) \dotp (y-x)} \leq \frac{\sm}{2} \norm{y-x}^2$ for all $x,y \in \reals^d$.}
We further assume that the function is lower bounded by some $f^\star \in \reals$ and denote $\fstar \eqdef f(w_1)-f^\star$ for a given reference point $w_1 \in \reals^d$. %
The goal in the setting is, given $T$ queries to the stochastic oracle and $\delta>0$, is to find a point $w \in \reals^d$ such that $\norm{\nabla f(w)}^2 \leq \epsilon$ with probability at least $1-\delta$, for $\epsilon$ as small as possible (as a function of the problem parameters $\smstar,\fstar,\sigmastar,\fsigma$ and $T, \delta$).

\item\emph{\bfseries Convex non-smooth setting.}
Here we assume $f$ is convex, $\lipstar$-Lipschitz and admits a minimizer $w^\star \in \argmin_{w \in \reals^d} f(w)$ with value $f^\star \eqdef f(w^\star)$.
We denote $\diamstar \eqdef \norm{w_1-w^\star}$ for a reference point $w_1 \in \reals^d$. %
The goal is, given $T$ queries to the stochastic oracle and $\delta>0$, is to produce a point $w \in \reals^d$ such that $f(w)-f^\star \leq \epsilon$ with probability at least $1-\delta$, for $\epsilon$ as small as possible (in terms of the problem parameters $\lipstar,\diamstar,\sigmastar,\fsigma$ and $T,\delta$).

\item\emph{\bfseries Convex and smooth setting.}
This scenario is identical to the former, but here we assume that $f$ is $\smstar$-smooth rather than $\lipstar$-Lipschitz. The convergence rate $\epsilon$ is then given in terms of the parameters $\smstar,\diamstar,\sigmastar,\fsigma$ and $T,\delta$.

\end{enumerate}

\subsection{Parameter-free algorithmic setup}

Since our focus is on parameter-free optimization, it is crucial to specify what algorithms in this setting are allowed to receive as input. 
First, we will always assume that the number of steps $T$, the failure probability $\delta$ and the reference point $w_1$ are given as inputs.%
\footnote{A standard doubling trick can be used to handle an unknown $T$ while keeping the same rate of convergence up to log-factors.}
The remaining problem parameters are not known ahead of time and thus are not received as inputs; however, we will assume that algorithms do receive, for each ground-truth parameter, a crude range (i.e., lower and upper bounds) in which it lies.
Namely, for any of the relevant parameters among $\smstar, \fstar, \lipstar, \diamstar, \sigmastar$, algorithms receive respective ranges such that
$\smstar \in [\smmin,\smmax], \fstar \in [\fmin,\fmax], \lipstar \in [\lipmin,\lipmax], \diamstar \in [\diammin,\diammax], \sigmastar \in [\sigmamin,\sigmamax]$ (a range for the parameter $\fsigma$ is not used in any of our methods).
We refer to a method as parameter-free, with respect to a well-tuned benchmark algorithm, if its convergence rate matches that of the benchmark, up to poly-logarithmic factors in the range parameters as well as in $T$ and $1/\delta$.
(We do not include a more formal definition here, since it will not be required for the statement of any of our results.)

\subsection{Tuned benchmarks}

Our benchmark rates, for both convex and non-convex optimization, are those of \emph{tuned} (fixed stepsize, possibly projected) Stochastic Gradient Descent (SGD). Starting at $w_1 \in \reals^d$, the update step of SGD is $w_{t+1}=w_{t} - \eta g_t$, where $\eta$ is the stepsize parameter, $g_t=\goracle(w_t)$ is the stochastic gradient at step $t$, evaluated at $w_t$. The update of projected SGD is $w_{t+1}=\proj{w_t-\eta g_t}$, where $\proj{\cdot}$ is the Euclidean projection onto a convex domain $\Omega \subseteq \reals^d$.
With a tuned stepsize parameter, this method achieves the best known rate in the non-convex setting and the optimal rate in convex non-smooth setting (accelerated methods achieve a slightly better rate in the convex and smooth setting).
The benchmark rates of SGD in the optimization scenarios described above are detailed in \cref{tab:sgd}.
\begin{table}[t]
    \centering
    \begin{tabular}{lll}
    \toprule
    \sc Setting &  \sc Benchmark rate & \\
    \midrule
    Non-convex & $\frac{\smstar \fstar}{T}+\frac{\sigmastar^2}{T}+\sqrt{\frac{\smstar\fstar\sigmastar^2}{T}}$ & \citep{ghadimi2013stochastic,liu2023high} \\
    Convex, Lipschitz & $\frac{\diamstar (\lipstar+\sigmastar)}{\sqrt{T}}$ & \citep{lan2012optimal} \\
    Convex, smooth & $\frac{\smstar \diamstar^2}{T}+\frac{\diamstar \sigmastar}{\sqrt{T}}$ & \citep{lan2012optimal} \\
    \bottomrule
    \end{tabular}
    \caption{Benchmark high-probability rates of tuned SGD, ignoring log-factors.}
    \label{tab:sgd}
\end{table}
Note that SGD is not parameter-free with respect to \emph{tuned} SGD, unless the range parameters are such that the ground-true parameters are known up to a constant factor.

\section{Parameter-Free Non-convex Optimization}
\label{sec:non-convex}
\begin{algorithm2e}[t]
    \SetAlgoLined
    \DontPrintSemicolon
    \KwIn{$w_1$, $T$, $\delta$, $\etamin$, $\etamax$ (see \cref{eq:non-convex-tuning} for replacing $\etamin,\etamax$ with parameter ranges)}
    $S \gets \set{}$ \;
    $K \gets \log\brk*{\frac{\etamax}{\etamin}\sqrt{\log_+(\tfrac{\etamax}{\etamin})\log_+(\tfrac{1}{\delta})}}$\;%
    $N \gets \log \tfrac{1}{\delta}$\;
    $T' \gets \frac{T}{K(1+N)}$\;
    \For{$k \gets 1,2,\dots,K$}{
        $\eta_k \gets 2^{k-1} \etamin$\;
        $w_1^{(k)},\ldots,w_T^{(k)} \gets \sgd\brk{w_1,\eta_k,T'}$\;
        Uniformly at random select $\wout_1^{(k)},\ldots,\wout_N^{(k)}$ from $w_1^{(k)},\ldots,w_T^{(k)}$\;
        $S \gets S \cup \set{\wout_1^{(k)},\ldots,\wout_n^{(k)}}$
    }
    \Return{$\argmin_{w \in S} \norm{\sum_{t=1}^{T'} \goracle_t(w)}$}
    \Comment*[r]{$\goracle_t(w)$ is an independent stochastic gradient at $w$}
    \caption{Non-convex SGD tuning} \label{alg:non-convex}
\end{algorithm2e}
We begin with the general non-convex smooth case, providing a fully parameter-free method which achieves the same rate of \emph{tuned} SGD up to poly-logarithmic factors.
Our approach is based on a simple observation that has also been used in the two-phase SGD approach of \citet{ghadimi2013stochastic}.
Our parameter-free method performs grid search over multiple SGD stepsizes and for each sequence of SGD iterations randomly samples a small portion of the iterations. Then the method picks the solution with the minimal approximate gradient norm based on stochastic gradients.
Note that \cref{alg:non-convex} receive $\etamin,\etamax$ instead of ranges $[\fmin,\fmax],[\smmin,\smmax],[\sigmamin,\sigmamax]$. This is done for framing the algorithm as a form of stepsize tuning (which in practice requires only two hyperparameters instead of two hyperparameters per problem parameter) and given the ranges the algorithm will use %
\begin{align}\label{eq:non-convex-tuning}
    \etamin &= \min\set*{\frac{1}{2\smmax},\sqrt{\frac{\fmin}{\smmax \sigmamax^2 T}}}
    \qquad
    \text{and}
    \qquad
    \etamax
    = \min\set*{\frac{1}{2\smmin},\sqrt{\frac{\fmax}{\smmin \sigmamin^2 T}}}.
\end{align}
Following is the convergence guarantee of the algorithm.

\begin{restatable}{theorem}{thmnonconvex}
\label{thm:non-convex}
    Assume that $f$ is $\smstar$-smooth and lower bounded by some $f^\star$ and $\goracle$ is a $\sigmastar$-bounded unbiased gradient oracle of $f$. Let $\etamin,\etamax > 0$ such that
    $$\etamin \leq \eta_\star \eqdef \min\set*{\tfrac{1}{2 \smstar}, \sqrt{\tfrac{\fstar}{\smstar \sigmastar^2 T}}} \leq \etamax$$
    where $\fstar = f(w_1)-f^\star$.
    Then for any $\delta \in (0,\tfrac13)$, given $w_1$, $T$, $\delta$, $\etamin$ and $\etamax$, \cref{alg:non-convex} performs $T$ gradient queries and produce $\wout$ such that with probability $\geq 1-\delta$,
    \begin{align*}
        \norm{\nabla f(\wout)}^2
        &
        =
        \Ohat\Bigg(
            \frac{\log_+(\tfrac{\etamax}{\etamin}) \log_+(\tfrac{1}{\delta})(\smstar \fstar
            + \sigmastar^2 \log \tfrac{1}{\delta})}{T}
            + \frac{\sigmastar \sqrt{\smstar \fstar \log_+ (\tfrac{\etamax}{\etamin}) \log_+(\tfrac{1}{\delta})}}{\sqrt{T}}
        \Bigg)
        .
    \end{align*}
\end{restatable}

First, we note that the rate of \cref{alg:non-convex} matches the rate of \emph{fully-tuned} SGD up to logarithmic factors.
Additionally, the method produces a single point instead of the obscure average norm guarantee of SGD, a product of comparing multiple iterations of SGD.
Comparing the result to recent advances in the analysis of adaptive SGD with AdaGrad-like stepsizes \citep{Faw2022ThePO,attia2023sgd,liu2023high}, \cref{alg:non-convex} has only logarithmic dependency in the specification of the parameters $\etamin$ and $\etamax$ while adaptive SGD requires tight knowledge of $\fstar/\smstar$ for tuning or suffer a sub-optimal polynomial dependency on problem parameters.
In addition, the proof technique of \cref{thm:non-convex} is simple and straightforward compared to the cumbersome analysis of adaptive SGD with AdaGrad-like stepsizes.

In order to prove the theorem we need the following lemmas
(see their proofs at \cref{sec:proofs-non-convex})%
.
The first is a standard convergence result of constant stepsize SGD. Alternatively, one can use the in-expectation analysis of SGD as in the two-phase SGD method of \citet{ghadimi2013stochastic}.
\begin{lemma}[SGD convergence with high probability]\label{lem:sgd-convergence}
    Assume that $f$ is $\smstar$-smooth and lower bounded by $f^\star$. Then for the iterates of SGD with $\eta \leq \ifrac{1}{2 \smstar}$, it holds with probability at least $1-\delta$ that
    \begin{align*}
        \frac1T \sum_{t=1}^T \norm{\nabla f(w_t)}^2
        &\leq
        \brk*{
            \frac{4 \fstar}{\eta}
            + 12 \sigmastar^2 \log \tfrac{1}{\delta}
        }
        \frac1T
        + 4 \smstar \sigmastar^2 \eta
        ,
    \end{align*}
    where $\fstar = f(w_1)-f^\star$.
\end{lemma}
Next is a standard success amplification technique for choosing between candidate solutions SGD; the precise version below is due to \citet{ghadimi2013stochastic}.
\begin{lemma}\label{lem:minimal-norm-candidate}
    Given candidates $w_1,\ldots,w_N$, let $\wout = \argmin_{n \in [N]} \norm{\sum_{t=1}^T \goracle_t(w_n)}$ where $g_t(w)$ for $t \in [T]$ are independent stochastic gradients at $w$.
    Then assuming 
    $\E[\goracle_t(w)]=\nabla f(w)$ and $\Pr(\norm{\nabla f(w)-\goracle_t(w)} \leq \sigma)=1$ 
    for any fixed $w$, for any $\delta \in (0,1)$, we have with probability at least $1-\delta$ that
    \begin{align*}
        \norm{\nabla f(\wout)}^2
        &\leq
        4 \min_{n \in [N]} \norm{\nabla f(w_n)}^2
        + \frac{24 (1+3 \log \tfrac{N}{\delta}) \sigma^2}{T}
        .
    \end{align*}
\end{lemma}

We proceed to prove the theorem.

\begin{proof}[Proof of \cref{thm:non-convex}]
    To fix the discrepancy between $T$ and $T'$, we denote
    \begin{align*}
        \eta_\star' \eqdef \min\set*{\frac{1}{2 \smstar}, \sqrt{\frac{\fstar \log_+(\tfrac{\etamax}{\etamin})\log_+(\tfrac{1}{\delta})}{\smstar \sigmastar^2 T}}} \geq \eta_\star,
    \end{align*}
    Note that $$\eta_1 \leq \eta_\star \leq \eta_\star' \leq \eta_\star \sqrt{\log_+(\tfrac{\etamax}{\etamin})\log_+(\tfrac{1}{\delta})} \leq \etamax \sqrt{\log_+(\tfrac{\etamax}{\etamin})\log_+(\tfrac{1}{\delta})} = 2\eta_K.$$
    Let $\tau = \max\set{k \in [K] : \eta_k \leq \eta_\star'}$. $\tau$ is well-defined as $\eta_1 \leq \eta_\star'$. Hence, since $2 \eta_K \geq \eta_\star'$, $\eta_\tau \in (\ifrac{\eta_\star'}{2},\eta_\star']$.
    Thus, by \cref{lem:sgd-convergence}, with probability at least $1-\delta$,
    \begin{align*}
        \frac{1}{T'} \sum_{t=1}^{T'} \norm{\nabla f(w_t^{(\tau)})}^2
        &\leq
        \brk*{
            \frac{4\fstar}{\eta_\tau}
            + 12 \sigmastar^2 \log \tfrac{1}{\delta}
        }
        \frac{1}{T'}
        + 4 \smstar \sigmastar^2 \eta_\tau
        \\&
        \leq
        \brk*{
            \frac{8\fstar}{\eta_\star'}
            + 12 \sigmastar^2 \log \tfrac{1}{\delta}
        }
        \frac{1}{T'}
        + 4 \smstar \sigmastar^2 \eta_\star'
        \\&
        \leq
        \frac{K(1+N)(16 \smstar \fstar
        + 12 \sigmastar^2 \log \tfrac{1}{\delta})}{T}
        + \frac{12 K(1+N)\sigmastar \sqrt{\smstar \fstar}}{\sqrt{T\log_+(\tfrac{\etamax}{\etamin})\log_+(\tfrac{1}{\delta})}}
        .
    \end{align*}
    By Markov's inequality, with probability at least $\frac12$, a uniformly at random index $k \in [T']$ satisfy
    \begin{align*}
        \norm{\nabla f(w_k^{(\tau)})}^2
        &\leq
        2 \E_k[\norm{\nabla f(w_k^{(\tau)})}^2]
        =
        \frac{2}{T'} \sum_{t=1}^T \norm{\nabla f(w_t^{(\tau)})}^2
        ,
    \end{align*}
    and as we sample $\log \tfrac{1}{\delta}$ random indices, with probability at least $1-\delta$ we add at least one point with the above guarantee to $S$.
    Hence, by a union bound with the convergence guarantee, with probability at least $1-2 \delta$ we add to $S$ some $w_k^{(\tau)}$ with
    \begin{align*}
        \norm{\nabla f(w_k^{(\tau)})}^2
        &\leq
        \frac{K(1+N)(32 \smstar \fstar
        + 24 \sigmastar^2 \log \tfrac{1}{\delta})}{T}
        + \frac{24 K(1+N)\sigmastar \sqrt{\smstar \fstar}}{\sqrt{T\log_+(\tfrac{\etamax}{\etamin})\log_+(\tfrac{1}{\delta})}}
        .
    \end{align*}
    The last step of the algorithm is an application of \cref{lem:minimal-norm-candidate} (the bounded noise assumption satisfy the sub-Gaussian requirement of \cref{lem:minimal-norm-candidate}) with the set of $KN$ candidates, and we conclude with a union bound that with probability at least $1-3\delta$, $\wout$, the output of \cref{alg:non-convex} satisfy
    \begin{align*}
        \norm{\nabla f(\wout)}^2
        &\leq
        \frac{K(1+N)(128 \smstar \fstar
        + 192 \sigmastar^2 \log \tfrac{KN}{\delta})}{T}
        + \frac{96 K(1+N)\sigmastar \sqrt{\smstar \fstar}}{\sqrt{T\log_+(\tfrac{\etamax}{\etamin})\log_+(\tfrac{1}{\delta})}}
        \\
        &=
        \Ohat\brk*{
            \frac{\log_+(\tfrac{\etamax}{\etamin}) \log_+(\tfrac{1}{\delta})(\smstar \fstar
            + \sigmastar^2 \log \tfrac{1}{\delta})}{T}
            + \frac{\sigmastar \sqrt{\smstar \fstar \log_+ (\tfrac{\etamax}{\etamin}) \log_+(\tfrac{1}{\delta})}}{\sqrt{T}}
        }
        .
    \end{align*}
    For each $k \in [K]$ we perform a single $T'$-steps SGD execution and $T' N$ gradient queries to approximate the norm of the randomly selected points, to a total of $T=T'K(1+N)$ queries. We translate the probability of $1-3\delta$ to $1-\delta$ for the $\Ohat$ notation, limiting $\delta \in (0,\tfrac13)$.
\end{proof}

\section{Parameter-Free Convex Optimization}
\label{sec:convex}

We proceed to consider convex problems, where the objective is either convex Lipschitz or convex smooth. The first result in this section assumes access via a first-order oracle (which includes access to noisy function values), while later results deal with the case of stochastic gradient oracle access.

\subsection{Optimization with a stochastic first-order oracle}
\label{sec:convex-zero-order}

In this setting we assume access using a first-order oracle $\oracle(w)=(\foracle(w),\goracle(w))$ with bounded noise.
Stochastic function value access (with a reasonable level of noise variance) is often a valid and natural assumption since (i) it is often the case that stochastic gradients are obtained by random samples and those can be used for both function and gradient evaluation; and (ii) the practice of tuning using a validation set implicitly assume this exact assumption.
We present a fully parameter-free method which utilize the function value oracle to perform model selection, similar to the use of the gradient oracle for model selection in \cref{alg:non-convex}.
The algorithm is detailed in
\cref{alg:convex-zero-order}
and following is its convergence result.

\begin{algorithm2e}[t]
    \SetAlgoLined
    \DontPrintSemicolon
    \KwIn{$w_1$, $T$, $\delta$, $\etamin$, $\etamax$
    (see Eq.~\ref{eq:zero-order-tuning} for replacing $\etamin,\etamax$ with parameter ranges)
    }
    $S \gets \set{}$\;
    $K \gets \log \tfrac{\etamax}{\etamin}$\;
    $N \gets \log \tfrac{1}{\delta}$\;
    $T' \gets \tfrac{T}{2 K N}$\;
    \For{$k \gets 1,2,\dots,K$}{
        $\eta_k \gets 2^{k-1} \etamin \sqrt{2KN}$\Comment*[r]{The $\sqrt{2KN}$ factor fix discrepancy between $T$ and $T'$}
        \For{$n \gets 1,2,\ldots,N$}{
            $\wout_k^{(n)} \gets \sgd\brk{w_1,\eta_k,T'}$\Comment*[r]{$\wout_k^{(n)}$ is the average of the iterates of \sgd}
            $S \gets S \cup \set{\wout_k^{(n)}}$\;
        }
    }
    \Return{$\argmin_{w \in S} \sum_{t=1}^{T'} \foracle_t(w)$}
    \Comment*[r]{$\foracle_t(w)$ is an independent evaluation of $\foracle(w)$}
    \caption{Convex SGD tuning} \label{alg:convex-zero-order}
\end{algorithm2e}

\begin{restatable}{theorem}{thmzeroorder}
\label{thm:convex-zero-order}
    Let $f : \reals^d \to \reals$ be a differentiable, $\lipstar$-Lipschitz and convex function, admitting a minimizer $w^\star \in \argmin_{w \in \reals^d} f(w)$. Let $\foracle$ be a $\fsigma$-bounded unbiased zero-oracle of $f$ and  $\goracle$ be a $\sigmastar$-bounded unbiased gradient oracle of $f$.
    Let $\etamin,\etamax > 0$ such that
    \begin{align*}
        \etamin \leq \frac{\norm{w_1-w^\star}}{\sqrt{(\lipstar^2+\sigmastar^2)T}} \leq \etamax
    \end{align*}
    for some $w_1 \in \reals^d$.
    Then \cref{alg:convex-zero-order} performs $T$ queries
    and produce $\wout$ such that with probability at least $1-2 \delta$,
    \begin{align*}
        f\brk*{\wout} & -f(w^\star)
        =
        \Ohat \Bigg(
        \frac{\fsigma \log (\tfrac{1}{\delta}) \sqrt{\log (\tfrac{\etamax}{\etamin})}}{\sqrt{T}}
        +
        \frac{\norm{w_1-w^\star}\sqrt{(\lipstar^2+\sigmastar^2)\log (\tfrac{\etamax}{\etamin}) \log (\tfrac{1}{\delta})}}{\sqrt{T}}
        \Bigg)
        .
    \end{align*}
\end{restatable}
\cref{alg:convex-zero-order} obtains the same rate of convergence as \emph{tuned} SGD up to logarithmic factors and a term of order $\ifrac{\fsigma}{\sqrt{T}}$, being parameter-free given the mild assumption that $\fsigma = O(\diamstar(\lipstar+\sigmastar))$.
Similarly to \cref{alg:non-convex}, the inputs $\etamin,\etamax$ may be replaced by problem parameters ranges
and setting
\begin{align}\label{eq:zero-order-tuning}
    \etamin &= \frac{\diammin}{\sqrt{(\lipmax^2 + \sigmamax^2) T}}
    \qquad
    \text{ and }
    \qquad
    \etamax = \frac{\diammax}{\sqrt{(\lipmin^2 + \sigmamin^2) T}}.
\end{align}
We note that \cref{alg:convex-zero-order} achieves an analogous guarantee for convex smooth objectives (via a proper setting of $\etamin,\etamax$ using $\diammin,\diammax,\smmin,\smmax$ and $\sigmamin,\sigmamax$).

Before proving \cref{thm:convex-zero-order}, we need the following standard lemmas
(for completeness we provide their proofs at \cref{sec:convex-proofs}).
First is an in-expectation convergence of SGD.
\begin{lemma}\label{lem:convex-sgd-expectation}
    Let $f : \reals^d \to \reals$ be a differentiable, $\lipstar$-Lipschitz and convex function, admitting a minimizer $w^\star \in \argmin_{w \in \reals^d} f(w)$ and let $\goracle : \reals^d \to \reals^d$ be a $\sigmastar$-bounded unbiased gradient oracle of $f$. Given some $w_1 \in \reals^d$, the iterates of $T$-steps SGD with stepsize $\eta>0$ satisfy
    \begin{align*}
        \E\brk[s]*{f\brk*{\frac1T \sum_{t=1}^T w_t}-f(w^\star)}
        &\leq
        \frac{\norm{w_1-w^\star}^2}{2 \eta T}
        + \frac{\eta (\lipstar^2 + \sigmastar^2)}{2}
        .
    \end{align*}
\end{lemma}
The second is a candidate selection lemma based on stochastic function value access.
\begin{lemma}\label{lem:minimal-function-candidate}
    Let $f : \reals^d \to \reals$ and $\foracle$ a zero-order oracle of $f$ such that
    \begin{align*}
        \forall ~ w \in \reals^d ~ : \E[\foracle(w)]=f(w) \quad \text{ and } \quad \Pr\brk{\abs{\foracle(w)-f(w)} \leq \fsigma}=1
    \end{align*}
    for some $\fsigma > 0$.
    Given candidates $w_1,\ldots,w_N$, let
    \begin{align*}
        \wout
        &=
        \argmin_{n \in [N]} \frac{1}{T} \sum_{t=1}^T \foracle_t(w_n),
    \end{align*}
    where $\foracle_t(w_n)$ for $t \in [T]$ are independent function evaluations at $w_n$. Then for any $\delta \in (0,1)$, with probability at least $1-\delta$,
    \begin{align*}
        f(\wout)
        \leq
        \min_{n \in [N]} f(w_s) + \sqrt{\frac{8 \fsigma^2 \log \tfrac{2 N}{\delta}}{T}}
        .
    \end{align*}
\end{lemma}
We proceed to prove the convergence result.
\begin{proof}[Proof of \cref{thm:convex-zero-order}]
    For each pair $(k,n)$, $T'$ gradient queries are made for the SGD and additional $T'$ function queries for the selection step. Hence, $T=2KNT'$ queries are performed.
    By the condition of $\etamin$ and $\etamax$, there is some $\tau \in [K]$ such that
    \begin{align*}
        \eta_\tau \in \bigg(\frac{\norm{w_1-w^\star}}{2\sqrt{(\lipstar^2+\sigmastar^2)T'}},\frac{\norm{w_1-w^\star}}{\sqrt{(\lipstar^2+\sigmastar^2)T'}}\bigg].
    \end{align*}
    Applying Markov's inequality to \cref{lem:convex-sgd-expectation} with $\eta_\tau$, for all $n \in [N]$, with probability at least $\frac12$,
    \begin{align*}
        f\brk*{\wout_\tau^{(n)}}-f(w^\star)
        &\leq
        \frac{\norm{w_1-w^\star}^2}{\eta_\tau T'}
        + \eta_\tau (\lipstar^2 + \sigmastar^2)
        \leq
        \frac{3 \norm{w_1-w^\star}\sqrt{\lipstar^2+\sigmastar^2}}{\sqrt{T'}}
        .
    \end{align*}
    Hence, by a union bound, with probability at least $1-\delta$, $S$ contains some $\wout_\tau^{(n)}$ with the above guarantee.
    We conclude by a union bound of the above with \cref{lem:minimal-function-candidate} and substituting $T'=\frac{T}{2 \log (\ifrac{\diammax}{\diammin})\log (\ifrac{1}{\delta})}$, such that with probability at least $1-2\delta$,
    \begin{align*}
        f\brk*{\wout}-f(w^\star)
        &\leq
        \frac{3 \norm{w_1-w^\star}\sqrt{(\lipstar^2+\sigmastar^2)\log (\tfrac{\etamax}{\etamin}) \log (\tfrac{1}{\delta})}}{\sqrt{T}}
        +\frac{\fsigma \sqrt{8 \log (\tfrac{2 K N}{\delta}) \log (\tfrac{1}{\delta}) \log (\tfrac{\etamax}{\etamin})}}{\sqrt{T}}
        \\
        &=
        \Ohat\brk*{\frac{\norm{w_1-w^\star}\sqrt{(\lipstar^2+\sigmastar^2)\log (\tfrac{\etamax}{\etamin}) \log (\tfrac{1}{\delta})}+\fsigma \log (\tfrac{1}{\delta}) \sqrt{\log (\tfrac{\etamax}{\etamin})}}{\sqrt{T}}}
        .\qedhere
    \end{align*}
\end{proof}

\subsection{Optimization with a stochastic gradient oracle}

While function value access is a reasonable assumption which leads to a simple parameter-free method, much of the existing work on adaptive, self-tuning optimization only assume access to stochastic gradient queries.
In this section we present a parameter-free method, using only stochastic gradient access, which match the convergence rate of \emph{tuned} SGD up to a lower-order term $O(\ifrac{\sigmamax}{T})$ and poly-logarithmic factors.
Our proposed method (\cref{alg:convex-lipschitz-new}) is based on tuning the diameter of the projected variant of Stochastic Gradient Descent with AdaGrad-like stepsizes method (\adapsgd). For a given diameter, the update step of \adasgd is
    $w_{t+1} = \projop_{w_1,\diam} \brk{w_t-\eta_t g_t}$,
where $\projop_{w_1,\diam}(\cdot)$ is the Euclidean projection onto the $\diam$-ball centered at the initialization $w_1$, $g_t=\goracle(w_t)$ is the stochastic gradient at step $t$, and for some $\initeta > 0$ and $\initg \geq 0$,
\begin{align*}
    \eta_t
    =
    \frac{\initeta \diam}{\sqrt{\initg^2 + \sum_{s=1}^t \norm{g_s}^2}}
    .
\end{align*}
Our method performs multiple runs of \adapsgd with an exponential grid of diameters and stops when all iterations are fully within the $\diam$-ball.
The method may be considered as a middle ground between \citet{carmon2022making} and \citet{Ivgi2023DoGIS} as it combines tuning and dynamic stepsizes.

\begin{algorithm2e}[t]
    \SetAlgoLined
    \DontPrintSemicolon
    \KwIn{$\diammin$, $\diammax$, $\lipmax$, $\sigmamax$, $T$ and $\delta$}
    $K \gets \log \tfrac{\diammax}{\diammin}$
    \;
    $\initg \gets \sqrt{5 \sigmamax^2 \log \tfrac{T}{\delta}}$\;
    $\theta_{T,\delta} \gets \log \tfrac{60 \log (6T)}{\delta}$\;
    $\initeta \gets 
    \bigg(
    48 \log \brk*{1 + \frac{2 \lipmax^2 T + 2 \sigmamax^2 T}{5 \sigmamax^2 \log \tfrac{T}{\delta}}}
    +
    32 \log \brk{1+T}
    +
    128 \sqrt{\theta_{T,\delta} \log \brk{1+T} + \frac{\theta_{T,\delta}^2}{5 \log \tfrac{T}{\delta}}}
    \bigg)
    ^{-1}$\;
    \For{$k \gets 1,2,\dots,K$}{
        $\diam_k \gets 2^{k+2} \diammin$\;
        $w_1^{(k)},\ldots,w_T^{(k)} \gets \adapsgd\brk{\diam_k,\initeta,\initg,w_1,T}$\;
        \If{$\max_{t \leq T} \norm{w_{t+1}^{(k)} - w_1} < \diam_k$}{
            \Return{$\bar{w}^{(k)} \gets \frac{1}{T} \sum_{t=1}^T w_t^{(k)}$}
        }
    }
    \Return{$w_1$} \Comment*[r]{Failure case}
    \caption{Adaptive projected SGD tuning} \label{alg:convex-lipschitz-new}
\end{algorithm2e}

Following is the convergence result of our method.

\begin{theorem}\label{thm:convex-lipschitz}
    Let $f : \reals^d \to \reals$ be a differentiable, $\lipstar$-Lipschitz and convex function, admitting a minimizer $w^\star \in \argmin_{w \in \reals^d} f(w)$. Let $\goracle$ be an unbiased gradient oracle of $f$ with $\sigmastar$-bounded noise, and let $w_1 \in \reals^d$.
    Then \cref{alg:convex-lipschitz-new} with parameters
    $\diammin \leq \norm{w_1-w^\star} < \diammax$, $\lipmax \geq \lipstar$ and $\sigmamax \geq \sigmastar$ performs at most $T \log \frac{\diammax}{\diammin}$ gradient queries and produce $\wout \in \reals^d$ such that with probability at least $1-\delta \log \tfrac{\diammax}{\diammin}$,
    \begin{align*}
        f&(\wout)-f(w^\star)
        \leq
        \bigg(
        \frac{\norm{w_1-w^\star}(\lipstar+\sigmastar)}{\sqrt{T}}
        +
        \frac{\norm{w_1-w^\star} \sigmamax}{T}
        \bigg)
        \polylog(\tfrac{\lipmax}{\sigmamax},T,\tfrac{1}{\delta})
        .
    \end{align*}
\end{theorem}
Observing the convergence rate in \cref{thm:convex-lipschitz}, \cref{alg:convex-lipschitz-new} achieve the same rate of convergence as \emph{fully-tuned} SGD up to logarithmic factors and an additional lower order term of $\Otilde(\ifrac{\sigmastar}{T})$.%
\footnote{In \cref{thm:lower-bound-new} we show that such a term is unavoidable without further assumptions and in \cref{sec:estimation} we provide a further assumption which mitigate the excess term.}
Hence, in case the ratio $\ifrac{\sigmamax}{(\sigmamin+\lipmin)}$ is $O(\sqrt{T})$, \cref{alg:convex-lipschitz-new} is parameter-free compared to the rate of \emph{tuned} SGD.
Previous parameter-free results for SCO \citep{carmon2022making,Ivgi2023DoGIS} suffer a similar lower order term, while our result depends only on the noise bound and not a bound of the norm of stochastic gradients which encapsulate both a noise and gradient norm bound, a property that stems from a careful combination of ``decorrelated'' stepsizes and martingale analysis instead of using the crude bound of stochastic gradients norm.
(This point will subsequently prove critical in the convex smooth case where gradients are not directly bounded.)
We also note that the bisection technique of \citet{carmon2022making} can improve the poly-logarithmic dependency; we abstain from using such a technique in favour of making the tuning procedure straightforward and leave improvements of the logarithmic factors to future work.

The convergence result of \cref{alg:convex-lipschitz-new} holds in fact holds more generally, assuming the gradient and noise bounds are satisfied only inside a ball of radius $8 \norm{w_1-w^\star}$. We will prove a more general version and \cref{thm:convex-lipschitz} will follow as a corollary. The refined result will be relevant later when we use the convergence guarantee for convex smooth objectives where a global gradient norm bound does not hold.
To that end we define $\lipball,\sigmaball \in \reals_+$ as the minimal parameters satisfying, for all $w \in \reals^d$ such that $\norm{w-w_1} \leq 8 \norm{w_1-w^\star}$,
\begin{align}
    \label{eq:lip-sigma-ball}
    &\norm{\nabla f(w)} \leq \lipball
    ,
    \enskip
    \Pr\brk{\norm{\goracle(w)-\nabla f(w)} \leq \sigmaball}
    =
    1
    .
\end{align}
Following is the general version of \cref{thm:convex-lipschitz}.

\begin{theorem}\label{thm:convex-ball}
    Let $f : \reals^d \to \reals$ be a differentiable and convex function, admitting a minimizer $w^\star \in \argmin f(w)$. Let $\goracle$ be an unbiased gradient oracle of $f$ such that $\sigmaball < \infty$ (see \cref{eq:lip-sigma-ball}), and let $w_1 \in \reals^d$.
    Then \cref{alg:convex-lipschitz-new} with parameters
    $\diammin \leq \norm{w_1-w^\star} < \diammax$, $\lipmax \geq \lipball$ and $\sigmamax \geq \sigmaball$, where $\lipball$ and $\sigmaball$ are defined by \cref{eq:lip-sigma-ball}, performs at most $T \log \frac{\diammax}{\diammin}$ gradient queries and produce $\wout_k=\frac1T \sum_{t=1}^T w_t^{(k)}$ for some $k \in [ \log \frac{\diammax}{\diammin}]$ such that with probability at least $1-\delta \log \tfrac{\diammax}{\diammin}$, $\diam_k \leq 8 \norm{w_1-w^\star}$ and
    \begin{align*}
        & f(\wout_k)-f^\star
        \leq
        \frac{1}{T} \sum_{t=1}^T f(w_t^{(k)})-f^\star
        =
        O\brk4{\frac{\widetilde C \norm{w_1-w^\star}}{T}\brk4{\sqrt{\sum_{t=1}^{T} \norm{g_t^{(k)}}^2}
        + \sigmamax \sqrt{\log \tfrac{T}{\delta}}}}
    \end{align*}
    where $\widetilde C = \log (\tfrac{T}{\delta}+\tfrac{\lipmax}{\sigmamax})$ and
    $g_t^{(k)}$ is the observed gradient at $w_t^{(k)}$.
\end{theorem}%
The bound of the general version uses the norms of the observed stochastic gradients instead of the gradient norm and noise bounds, and a simple translation
yields \cref{thm:convex-lipschitz}.
\begin{proof}[Proof of \cref{thm:convex-lipschitz}]%
    Using the identity $\norm{u+v}^2 \leq 2 \norm{u}^2 + 2 \norm{v}^2$ and the global bounds, for all $w \in \reals^d$,
    \begin{align*}
        \norm{\goracle(w)} \leq 2 \norm{\nabla f(w)}^2 + 2 \norm{\goracle(w)-\nabla f(w)}^2 \leq 2 \lipstar^2 + 2 \sigmastar^2.
    \end{align*}
    By the minimality of $\lipball$ and $\sigmaball$, $\lipmax \geq \lipstar \geq \lipball$ and $\sigmamax \geq \sigmastar \geq \sigmaball$, which means that \cref{thm:convex-ball} is applicable.
    Combining \cref{thm:convex-ball} with the above inequality, \cref{alg:convex-lipschitz-new} produce $\wout_k$ such that with probability at least $1-\delta \log \tfrac{\diammax}{\diammin}$,
    \begin{align*}
        f(\wout_k)-f^\star
        &=
        O\brk4{\frac{\widetilde C \norm{w_1-w^\star}}{T}\brk4{\sqrt{(\lipstar^2+\sigmastar^2)T}
        + \sigmamax \sqrt{\log \tfrac{T}{\delta}}}}
        .
    \end{align*}
    We conclude by moving from big-$O$ notation to poly-log notation, replacing the term $\widetilde C \sqrt{\log \tfrac{T}{\delta}}$ with $\polylog(\tfrac{\lipmax}{\sigmamax},T,\tfrac{1}{\delta})$.
\end{proof}

In order to prove \cref{thm:convex-ball}, we first provide several intermediate results for $T$-steps \adapsgd for some $\diam \leq 8 \norm{w_1-w^\star}$ at \cref{lem:padasgd-convergence,lem:padasgd-iterate-bound,lem:padasgd-convergence-martingale} (in which case the bounds $\lipball$ and $\sigmaball$ hold for all the projected domain).
We defer their proofs to \cref{sec:convex-proofs}.
Differently from common analyses of projected methods, we are interested in both the local minimizer inside the $\diam$-ball and the global minimizer in $\reals^d$. To that end, let
$
    w^\star_\diam = \argmin_{w : \norm{w-w_1} \leq \diam} f(w)
$
such that $\norm{w_1-w^\star_\diam} \leq \norm{w_1-w^\star}$.\footnote{Such $w^\star_\diam$ exists since either $\norm{w_1-w^\star} \leq \diam$ and we may use $w^\star_\diam=w^\star$, or $\norm{w_1-w^\star_\diam} \leq \diam < \norm{w_1-w^\star}$.}
The first lemma establishes a convergence guarantee with respect to the \emph{global minimizer} given the following two events. The first is that all the iterates are contained (fully) within the $\diam$-ball, i.e., $\event_1 \eqdef \set{\max_{t \leq T} \norm{w_{t+1}-w_1} < \diam}$. The second,
\begin{align*}
    \event_2 & \eqdef
    \Bigg\{%
        \abs*{\sum_{t=1}^T (\nabla f(w_t)-g_t) \cdot (w_t-w^\star)}
        \leq
        4 (\diam+\norm{w_1-w^\star}) \sqrt{\theta_{T,\delta} \sum_{t=1}^T \norm{g_t}^2 + 4 \theta_{T,\delta}^2 \sigmaball^2}
    \Bigg\}%
    ,
\end{align*}
is the concentration needed to modify the regret analysis to a high-probability analysis.
\begin{lemma}\label{lem:padasgd-convergence}
    Under the event $\event_1 \cap \event_2$,
    \begin{align*}
        \frac1T \sum_{t=1}^T f(w_t)-f(w^\star)
        &
        \leq
        \brk*{\tfrac{2 \initg}{\initeta}
        + 8 \theta_{T,\delta} \sigmaball}\frac{\norm{w_1-w^\star}+\diam}{T}
        \\
        &
        + \brk*{\brk{\norm{w_1-w^\star}+\diam}\brk{\tfrac{2}{\initeta}+4 \sqrt{\theta_{T,\delta}}} + \initeta \diam}\frac{\sqrt{\sum_{t=1}^T \norm{g_t}^2}}{T}.
    \end{align*}
\end{lemma}
The second lemma bounds the distance of the iterates from the initialization.
This lemma will prove useful to show that $\event_1$ holds with high probability for a large enough $\diam$.
\begin{lemma}\label{lem:padasgd-iterate-bound}
    Let $\initeta,\initg$ as defined by \cref{alg:convex-lipschitz-new}.
    Then with probability at least $1-2 \delta$, for all $t \in [T]$,
    \begin{align*}
        \norm{w_{t+1}-w_1}
        \leq
        2 \norm{w_1-w^\star}
        + \frac{1}{2} \diam
        .
    \end{align*}
\end{lemma}
The third lemma lower bounds the probability of $\event_2$. Note that the concentration bound does not depends on the gradient norm bound, only the observed (noisy) gradients and the noise bound.
\begin{lemma}\label{lem:padasgd-convergence-martingale}
    With probability at least $1-2\delta$,
    \begin{align*}
        & \abs*{\sum_{t=1}^T (\nabla f(w_t)-g_t) \cdot (w_t-w^\star)}
        \leq
        4 (\diam+\norm{w_1-w^\star}) \sqrt{\theta_{T,\delta} \sum_{t=1}^T \norm{g_t}^2 + 4 \theta_{T,\delta}^2 \sigmaball^2}
        ,
    \end{align*}
    where $\theta_{T,\delta}=\log \tfrac{60 \log (6T)}{\delta}$. In other words, 
    $\Pr(\event_2) \geq 1-2\delta$.
\end{lemma}
We are ready to prove the main result.
\begin{proof}[Proof of \cref{thm:convex-ball}]%
    Let $\wout^{(\tau)}$ for $\tau \in [K+1]$ be the output of \cref{alg:convex-lipschitz-new} (treating $w_1$ as $\wout^{(K+1)}$).
    Let
    $$i = \min\set{k \in [K] : \diam_k > 4 \norm{w_1-w^\star}}.$$
    Note that $i$ is well defined as $\diam_K=4 \diammax > 4 \norm{w_1-w^\star}$ and that by minimality of $i$, $\diam_i \leq 8 \norm{w_1-w^\star}$.
    We define the events $\event_1^{(k)}$ and $\event_2^{(k)}$ as the respective $\event_1$ and $\event_2$ events for the $k$'th sequence of \adapsgd, $w_1^{(k)},\ldots,w_T^{(k)}$.
    From \cref{lem:padasgd-iterate-bound}, with probability at least $1-2\delta$, for all $t \in [T]$,
    \begin{align*}
        \norm{w_{t+1}^{(i)}-w_1}
        \leq
        2 \norm{w_1-w^\star}
        + \frac12 \diam_i
        < \diam_i,
    \end{align*}
    where the last inequality follows by $\diam_i > 4 \norm{w_1-w^\star}$. Hence, $\Pr\set{\tau \leq i} \geq \Pr\set{\tau=i \mid \tau \geq i} \geq 1-2\delta$.
    By \cref{lem:padasgd-convergence-martingale}, with probability at least $1-2 \delta K$ (union bound over all $k \in [i]$, where $i \leq K$), $\event_2^{(k)}$ holds for all $k \in [i]$.
    Again using union bound with $\set{\tau \leq i}$, with probability at least $1-2\delta(K+1)$, we also have that $\set{\tau \leq i}$, which means that both $\event_1^{(\tau)}$ and $\event_2^{(\tau)}$ hold and we can invoke \cref{lem:padasgd-convergence}.
    Hence, with probability $1-2 \delta (K+1)$, \cref{alg:convex-lipschitz-new} returns $\wout^{(\tau)}$ for some $\tau \leq i$ for which
    \begin{align*}
        \frac1T \sum_{t=1}^T f(w_t^{(\tau)})-f^\star
        &
        \leq 
        \brk*{\tfrac{2 \initg}{\initeta}
        + 8 \theta_{T,\delta} \sigmaball}\frac{\norm{w_1-w^\star}+\diam_\tau}{T}
        \\
        &
        \! + \! \brk*{\brk{\norm{w_1-w^\star}+\diam_\tau}\brk{\tfrac{2}{\initeta}+4 \sqrt{\theta_{T,\delta}}} + \initeta \diam_\tau}\frac{\sqrt{\sum_{t=1}^T \norm{g_t}^2}}{T}
        .
    \end{align*}
    Using the inequalities $\diam_\tau \leq \diam_i \leq 8 \norm{w_1-w^\star}$, $\sigmaball \leq \sigmamax$, and a standard application of Jensen's inequality,
        \begin{align*}
        & f(\wout^{(\tau)}) - f^\star
        =
        O\bigg(
        \brk*{\tfrac{\initg}{\initeta \sigmamax}
        + \theta_{T,\delta}}\frac{\norm{w_1-w^\star} \sigmamax}{T}
        + \brk*{\brk{\tfrac{1}{\initeta}+\sqrt{\theta_{T,\delta}}} + \initeta}\frac{\norm{w_1-w^\star}\sqrt{\sum_{t=1}^T \norm{g_t^{(\tau)}}^2}}{T}
        \bigg)
        .
    \end{align*}
    We conclude by substituting $\initg$, noting that $\alpha \leq 1$, $\sqrt{\theta_{T,\delta}},\theta_{T,\delta},\alpha^{-1} = O\brk!{\log \brk!{\tfrac{T}{\delta}+\tfrac{\lipmax}{\sigmamax}}}$.
\end{proof}

\subsection{Convex and smooth stochastic optimization}

In this setting we assume that the objective is smooth rather then Lipschitz. We prove that \cref{alg:convex-lipschitz-new} also achieve the same rate of convergence as \emph{tuned} SGD for smooth objectives, up to logarithmic factors and a lower-order term.
\begin{theorem}\label{thm:convex-smooth}
    Let $f : \reals^d \to \reals$ be a $\smstar$-smooth and convex function, admitting a minimizer $w^\star \in \argmin f(w)$. Let $\goracle$ be an unbiased gradient oracle of $f$ and let $w_1 \in \reals^d$.
    Then \cref{alg:convex-lipschitz-new} with parameters
    $\diammin \leq \norm{w_1-w^\star} < \diammax$, $\lipmax \geq 9 \smstar \norm{w_1-w^\star}$ and $\sigmamax \geq \sigmaball$, where $\sigmaball$ is defined by \cref{eq:lip-sigma-ball}, performs at most $T \log \frac{\diammax}{\diammin}$ gradient queries and produce $\wout$ such that with probability at least $1-\delta \log \tfrac{\diammax}{\diammin}$, $\diam_k \leq 8 \norm{w_1-w^\star}$ and
    \begin{align*}
        f(\wout)-f^\star
        &
        =
        O\Bigg(
        \frac{\widetilde C^2 \smstar \norm{w_1-w^\star}^2}{T}
        + \frac{\widetilde C \norm{w_1-w^\star} \sigmaball}{\sqrt{T}}
        + \sigmamax \frac{\widetilde C \norm{w_1-w^\star}\sqrt{\log \tfrac{T}{\delta}}}{T}
        \Bigg)
        ,
    \end{align*}
    where $\widetilde C = \log (\tfrac{T}{\delta}+\tfrac{\lipmax}{\sigmamax})$.
\end{theorem}%
The method is parameter-free compared to \emph{tuned} SGD, assuming $\ifrac{\sigmamax}{\sigmamin}=O(\sqrt{T})$.
Previous parameter-free methods for stochastic convex optimization \citep{carmon2022making,Ivgi2023DoGIS} provided only a noiseless guarantee for smooth objectives.
\citet{attia2023sgd} provided a convergence result for adaptive SGD with AdaGrad-like stepsizes for convex and smooth objectives, but similarly to the non-convex case, the method is not parameter-free compared to \emph{tuned} SGD.

\begin{proof}[Proof of \cref{thm:convex-smooth}]
    In order to show that \cref{thm:convex-ball} is applicable we need to show that $\lipmax \geq \lipball$. Let $w \in \reals^d$ such that $\norm{w-w_1} \leq 8 \norm{w_1-w^\star}$. Thus, by the smoothness of $f$ and the triangle inequality,
    $$
    \norm{\nabla f(w)}
    \leq
    \smstar \norm{w-w^\star}
    \leq
    \smstar (\norm{w-w_1}+\norm{w_1-w^\star})
    \leq 9 \smstar \norm{w_1-w^\star}
    \leq
    \lipmax.
    $$
    Hence, by \cref{eq:lip-sigma-ball}, $\lipmax \geq \lipball$.
    Using \cref{thm:convex-ball}, with probability at least $1-\delta \log \tfrac{\diammax}{\diammin}$, we have a sequence $w_1,\ldots,w_T$ with
    \begin{align}\label{eq:convex-smooth-pre}
        \frac{1}{T} \sum_{t=1}^T f(w_t)-f^\star
        &
        \leq
        C \brk*{\frac{\log (\tfrac{T}{\delta}+\tfrac{\lipmax}{\sigmamax}) \norm{w_1-w^\star}}{T} \brk*{
        \sqrt{\sum_{t=1}^{T} \norm{g_t}^2}
        + \sigmamax \sqrt{\log \tfrac{T}{\delta}}
        }}
    \end{align}
    for some $C > 0$, and for all $t$, $\norm{w_t-w_1} \leq 8 \norm{w_1-w^\star}$.
    By the identity $\norm{a+b}^2 \leq 2 \norm{a}^2 + 2 \norm{b}^2$ and \cref{eq:lip-sigma-ball},
    \begin{align*}
        \norm{g_t}^2
        &\leq
        2 \norm{\nabla f(w_t)}^2
        + 2 \norm{g_t-\nabla f(w_t)}^2
        \leq
        2 \norm{\nabla f(w_t)}^2
        + 2 \sigmaball^2
        \leq
        4 \smstar (f(w_t)-f^\star)
        + 2 \sigmaball^2
        .
    \end{align*}
    where the last inequality is due to the following standard smoothness argument. For any $w \in \reals^d$ and $w^+=w-\frac{1}{\smstar} \nabla f(w)$, by the smoothness of $f$,
    \begin{align*}
        2 \smstar (f(w)-f^\star)
        &\geq 2 \smstar (f(w)-f(w^+))
        \geq 2 \smstar \brk{-\nabla f(w) \cdot (w^+-w)-\tfrac{\smstar}{2} \norm{w^+-w}^2}
        =
        \norm{\nabla f(w)}^2
        .
    \end{align*}
    Hence, by the identities $\sqrt{a+b} \leq \sqrt{a}+\sqrt{b}$ and $ab \leq \frac{1}{2 \lambda} a^2 + \frac{\lambda}{2} b^2$ for $\lambda > 0$,
    \begin{align*}
        \sqrt{\sum_{t=1}^T \norm{g_t}^2}
        &\leq
        \sqrt{4 \smstar \sum_{t=1}^T \brk{f(w_t)-f^\star}}
        + \sigmaball \sqrt{2T}
        \leq
        \frac{1}{2\lambda} \sum_{t=1}^T \brk{f(w_t)-f^\star}
        + 2 \lambda \smstar
        + \sigmaball \sqrt{2T}
        .
    \end{align*}
    Setting $\lambda=C \norm{w_1-w^\star} \log (\tfrac{T}{\delta}+\tfrac{\lipmax}{\sigmamax})$, returning to \cref{eq:convex-smooth-pre} and rearranging,
    \begin{align*}
        \frac{1}{T} \sum_{t=1}^T f(w_t)-f^\star
        &
        \leq
        \frac{4 C^2 \smstar \norm{w_1-w^\star}^2 \log^2 (\tfrac{T}{\delta}+\tfrac{\lipmax}{\sigmamax})}{T}
        + \frac{\sqrt{8} C \norm{w_1-w^\star} \sigmaball \log (\tfrac{T}{\delta}+\tfrac{\lipmax}{\sigmamax})}{\sqrt{T}}
        \\&
        + \sigmamax \frac{2 C \norm{w_1-w^\star} \log (\tfrac{T}{\delta}+\tfrac{\lipmax}{\sigmamax}) \sqrt{\log \tfrac{T}{\delta}}}{T}
        .
    \end{align*}
    We conclude by a standard application of Jensen's inequality.
\end{proof}

\subsection{Information-theoretic limits to parameter freeness}

Both \cref{thm:convex-ball,thm:convex-smooth} suffer an additional term of $\Otilde(\ifrac{\norm{w_1-w^\star} \sigmamax}{T})$ compared to tuned SGD, that prevents them from being fully parameter-free compared to SGD due to the polynomial dependence on $\sigmamax$. 
The following theorem shows that without additional assumptions, this term is in fact unavoidable and the results of \cref{thm:convex-ball,thm:convex-smooth} are essentially the best one can hope for.

\begin{theorem}\label{thm:lower-bound-new}
    Let $T \geq 4$, $\diammax>0$, $\sigmamax>0$, $\alpha \in [1,\frac34 \sqrt{T}]$ and $\lipstar=\frac{\sigmamax}{2T-1}$. Let $\alg$ be any deterministic algorithm which performs $T$ gradient queries and outputs $\wout \in \reals$. Then there exist some $\diamstar \in [\frac{\diammax}{\alpha \sqrt{T}} ,\diammax]$, $\sigmastar \in [\frac{\sigmamax}{\alpha \sqrt{T}},\sigmamax]$, a convex and $\lipstar$-Lipschitz function $f : \reals \to \reals$ which admits a minimizer $w^\star \in \argmin_{w \in \reals} f(w)$ such that $\norm{w^\star} = \diamstar$, and a $\sigmastar$-bounded unbiased sub-gradient\footnote{The function is of the form $\lipstar\abs{w-\diamstar}$, and smoothing can be used to support smooth objectives. With a slight abuse of notation we treat $\nabla f(w)$ as some sub-gradient of $f$ such that $\E[\goracle(w)]=\nabla f(w)$.}
    oracle of $f$, $\goracle$ such that
    with probability at least $\frac14$,
    \begin{align*}
        f(\wout)-f(w^\star)
        &\geq
        \frac{\diamstar (\lipstar+\sigmastar) \alpha}{6 \sqrt{T}}
        \geq
        \frac{\diamstar \sigmamax}{6 T}
        .
    \end{align*}
\end{theorem}
We remark that the deterministic requirement is made for simplicity and the argument can be adapted to randomized algorithms.
To understand the lower bound consider the convergence guarantee of tuned SGD. A standard in-expectation analysis of SGD with a \emph{tuned} stepsize of $\eta=\ifrac{\diam}{(\lip+\sigma) \sqrt{T}}$ and an application of Markov's inequality shows that with probability at least $\frac34$, SGD produce $\wout$ with
\begin{align*}
    f(\wout)-f(w^\star)
    =
    O\brk*{\frac{\diamstar (\lipstar+\sigmastar)}{\sqrt{T}}}
    .
\end{align*}
Hence, when $\alpha = \omega(1)$, the worst case analysis of SGD is strictly better than any algorithm with the limited knowledge prescribed in \cref{thm:lower-bound-new}. In particular, a method cannot be parameter-free and expect a comparable rate of convergence to SGD unless one of $[\diammin,\diammax],[\sigmamin,\sigmamax]$ has ratio of $\Otilde(\sqrt{T})$ or the $\lipstar$ value is outside the range (for example if the noise bound is small).
This lower bound of \cref{thm:lower-bound-new} indeed appears in the guarantees of \citet{carmon2022making,Ivgi2023DoGIS} in the form of the bound of stochastic gradients norm and in our \cref{thm:convex-ball}.
While the term $\Omega\brk{\ifrac{\norm{w_1-w^\star} \sigmamax}{T}}$ is unavoidable, in \cref{sec:estimation} we show that assuming the initialization is a ``bad enough'' approximate minimizer, we can infer a reasonable bound of $\sigmastar$ which reduce the lower bound to a term proportional to the guarantee of tuned SGD.

\begin{proof}[Proof of \cref{thm:lower-bound-new}]
    We will define two problem instances $(f_1,\goracle_1,\diam_1,\lip_1),(f_2,\goracle_2,\diam_2,\lip_2)$, which will be indistinguishable with constant probability using only $T$ gradient queries, where no solution $\wout \in \reals$ can satisfy the convergence bound for both.
    Let $\sigma_1 = \frac{\sigmamax}{\alpha \sqrt{T}}$, $\diam_1=\diammax$, $\sigma_2 = \sigmamax$, $\diam_2=\frac{\diammax}{\alpha \sqrt{T}}$. We define the two functions,
    \begin{align*}
        f_1(w) = \lipstar \abs{w-\diam_1}
        \qquad \text{and} \qquad
        f_2(w) = \lipstar \abs{w+\diam_2}.
    \end{align*}
    Note that both functions are $\lipstar$-Lipschitz and their minimizers satisfy $\norm{\argmin_{w \in \reals} f_i(w)} \leq \diam_i$ for $i \in \set{1,2}$.
    In addition, we define the following two first-order oracles,
    \begin{align*}
        \goracle_1(w)
        &=
        \begin{cases}
            - \lipstar & \text{if $w \leq -\diam_2$}; \\
            \lipstar & \text{if $w \geq \diam_1$}; \\
            - \sigma_1 + \lipstar & \text{if $w \in (-\diam_2,\diam_1)$, w.p. } \tfrac12+\tfrac{\lipstar}{2(\sigma_1-\lipstar)}; \\
            \sigma_1-\lipstar & \text{if $w \in (-\diam_2,\diam_1)$, w.p. } \tfrac12-\tfrac{\lipstar}{2(\sigma_1-\lipstar)}
        \end{cases}
    \end{align*}
    and
    \begin{align*}
        \goracle_2(w)
        &=
        \begin{cases}
            - \lipstar & \text{if $w \leq -\diam_2$}; \\
            \lipstar & \text{if $w \geq \diam_1$}; \\
            \sigma_2+\lipstar & \text{if $w \in (-\diam_2,\diam_1)$, w.p. } \tfrac1T; \\
            - \sigma_1 + \lipstar & \text{if $w \in (-\diam_2,\diam_1)$, w.p. } \tfrac12+\tfrac{\lipstar}{2(\sigma_1-\lipstar)} -\frac1{2T}; \\
            \sigma_1 - \lipstar & \text{if $w \in (-\diam_2,\diam_1)$, w.p. } \tfrac12-\tfrac{\lipstar}{2(\sigma_1-\lipstar)} -\tfrac1{2T}.
        \end{cases}
    \end{align*}
    Note that the probabilities are within $[0,1]$ (and sum to $1$) since $\frac{1}{2T} \leq \frac{1}{8}$ as $T \geq 4$, and for $\alpha \in [1,\frac34 \sqrt{T}]$,
    \begin{align*}
        \frac{\lipstar}{2 (\sigma_1-\lipstar)}
        &= \frac{\sigmamax / (2T-1)}{2(\sigmamax/\alpha \sqrt{T} - \sigmamax/(2T-1))}
        = \frac{\alpha \sqrt{T}}{2(2T-1-\alpha \sqrt{T})}
        \leq
        \frac{3T}{10T-8}
        \leq \frac{3}{8}
        .
    \end{align*}
    Additionally, for $w \in (-\diam_2,\diam_1)$,
    \begin{align*}
        \E[\goracle_1(w)]
        =
        (\sigma_1-\lipstar)\brk*{\frac12-\frac{\lipstar}{2(\sigma_1-\lipstar)}-\frac12-\frac{\lipstar}{2(\sigma_1-\lipstar)}}
        =
        -\frac{\lipstar(\sigma_1-\lipstar)}{\sigma_1-\lipstar}
        = - \lipstar
        = \nabla f_1(w)
    \end{align*}
    and similarly
    \begin{align*}
        \E[\goracle_2(w)]
        &=
        \frac{\sigma_2 + \lipstar}{T}
        - \lipstar
        =
        \frac{(2 T -1 ) \lipstar + \lipstar}{T}
        - \lipstar
        = \lipstar
        = \nabla f_2(w)
        .
    \end{align*}
    Hence, for all $w \in \reals$,
    $\E[\goracle_1(w)]=\nabla f_1(w)$, $\E[\goracle_2(w)]=\nabla f_2(w)$, and with probability $1$, as $\lipstar < \sigma_1 < \sigma_2$,
    \begin{align*}
        \norm{\goracle_1(w)-\nabla f_1(w)} &\leq
        \sigma_1
        \quad \text{and} \quad
        \norm{\goracle_2(w)-\nabla f_2(w)} \leq
        \sigma_2.
    \end{align*}
    We concluded the properties of $f$ and $\goracle$ and we move to the lower bound.
    For $n \geq 4$,
    \begin{align*}
        \brk*{1-\frac{1}{n}}^n
        =
        \brk*{\frac{n}{n-1}}^{-n}
        \geq
        \frac{3}{4} \brk*{\frac{n}{n-1}}^{-(n-1)}
        =
        \frac{3}{4} \frac{1}{\brk*{1+\frac{1}{n-1}}^{n-1}}
        \geq
        \frac{1}{4}
        ,
    \end{align*}
    where the last inequality follows by
    \begin{align*}
        \brk*{1+\frac{1}{n-1}}^{n-1} = 2 + \sum_{k=2}^{n-1} \binom{n-1}{k} (n-1)^{-k} \leq 2 + \sum_{k=2}^{n-1} \frac{1}{k(k-1)}
        = 2 + \sum_{k=2}^{n-1} \brk*{\frac{1}{k-1}-\frac{1}{k}} \leq 3.
    \end{align*}
    Hence, with probability at least $\frac14$, the $\ifrac{1}{T}$ event of $\goracle_2$ do not occur in any of the $T$ queries and the two gradient oracles will be indistinguishable from each other.
    Let $\wout$ be the output of $\alg(T)$ given that the two oracles return the exact same gradients.
    We will assume by contradiction that $\wout$ satisfy a convergence rate better than $\ifrac{\diam_i (\lipstar+\sigma_i)\alpha}{6\sqrt{T}}$ for each $f_i$ ($i \in \set{1,2}$ respectively). Thus, as $\sigma_1=\frac{\lipstar(2T-1)}{\alpha \sqrt{T}}$ and $\alpha < \sqrt{T}$,
    \begin{align*}
        f_1(\wout)-f_1(w^\star)
        &=
        \lipstar \abs{\wout-\diam_1}
        =
        \frac{\diam_1 (\lipstar+\sigma_1) \alpha}{\sqrt{T}} \cdot \frac{\lipstar \sqrt{T} \abs{\wout-\diam_1}}{\diam_1 (\lipstar+\sigma_1)\alpha}
        \\&
        =
        \frac{\diam_1 (\lipstar+\sigma_1) \alpha}{\sqrt{T}} \cdot \frac{T \abs{\wout-\diam_1}}{\diam_1 (\alpha \sqrt{T} + 2T - 1)}
        \geq
        \frac{\diam_1 (\lipstar+\sigma_1) \alpha}{\sqrt{T}} \cdot \frac{\abs{\wout-\diam_1}}{3 \diam_1}
    \end{align*}
    which implies that $\wout > \ifrac{\diam_1}{2}$ by the convergence assumption.
    Similarly, as $\sigma_2 = \lipstar (2T-1)$,
    \begin{align*}
        f_2(\wout)-f_2(w^\star)
        &=
        \lipstar \abs{\wout+\diam_2}
        =
        \frac{\diam_2 (\lipstar+\sigma_2) \alpha}{\sqrt{T}} \cdot \frac{\lipstar \sqrt{T} \abs{\wout+\diam_2}}{\diam_2(\lipstar + \sigma_2) \alpha}
        =
        \frac{\diam_2 (\lipstar+\sigma_2) \alpha}{\sqrt{T}} \cdot \frac{\abs{\wout+\diam_2}}{2 \diam_2 \sqrt{T} \alpha}
    \end{align*}
    which implies that $\wout < (\alpha \sqrt{T} / 3 - 1) \diam_2$ by the convergence assumption. But
    \begin{align*}
        \wout > \frac{\diam_1}{2} = \frac{\alpha \sqrt{T} \diam_2}{2} > (\alpha \sqrt{T} / 3-1)\diam_2 > \wout
    \end{align*}
    and we obtain a contradiction. The second inequality of the convergence lower bound follows by
    $(\lipstar+\sigma)\alpha \geq \sigmamax (\frac{\alpha}{2T-1}+\frac{1}{\sqrt{T}})\geq \frac{\sigmamax}{\sqrt{T}}$ for  $\sigma \in \set{\sigma_1,\sigma_2}$.
\end{proof}

\subsection*{Acknowledgements}
We are grateful to Yair Carmon for invaluable comments and discussions.
This project has received funding from the European Research Council (ERC) under the European
Union’s Horizon 2020 research and innovation program (grant agreement No. 101078075).
Views and opinions expressed are however those of the author(s) only and do not necessarily reflect
those of the European Union or the European Research Council. Neither the European Union nor
the granting authority can be held responsible for them.
This work received additional support from the Israel Science Foundation (ISF, grant number 2549/19), from the Len Blavatnik and the Blavatnik Family foundation, from the Adelis Foundation, and from the Prof.\ Amnon Shashua and Mrs.\ Anat Ramaty Shashua Foundation.

\bibliographystyle{abbrvnat}
\bibliography{references}

\begin{thebibliography}{47}
\providecommand{\natexlab}[1]{#1}
\providecommand{\url}[1]{\texttt{#1}}
\expandafter\ifx\csname urlstyle\endcsname\relax
  \providecommand{\doi}[1]{doi: #1}\else
  \providecommand{\doi}{doi: \begingroup \urlstyle{rm}\Url}\fi

\bibitem[Alacaoglu et~al.(2020)Alacaoglu, Malitsky, Mertikopoulos, and Cevher]{alacaoglu2020new}
A.~Alacaoglu, Y.~Malitsky, P.~Mertikopoulos, and V.~Cevher.
\newblock A new regret analysis for adam-type algorithms.
\newblock In \emph{International conference on machine learning}, pages 202--210. PMLR, 2020.

\bibitem[Arjevani et~al.(2022)Arjevani, Carmon, Duchi, Foster, Srebro, and Woodworth]{arjevani2022lower}
Y.~Arjevani, Y.~Carmon, J.~C. Duchi, D.~J. Foster, N.~Srebro, and B.~Woodworth.
\newblock Lower bounds for non-convex stochastic optimization.
\newblock \emph{Mathematical Programming}, pages 1--50, 2022.

\bibitem[Attia and Koren(2023)]{attia2023sgd}
A.~Attia and T.~Koren.
\newblock Sgd with adagrad stepsizes: Full adaptivity with high probability to unknown parameters, unbounded gradients and affine variance.
\newblock In \emph{International Conference on Machine Learning}, 2023.

\bibitem[Attia and Koren(2024{\natexlab{a}})]{attia2024free}
A.~Attia and T.~Koren.
\newblock How free is parameter-free stochastic optimization?
\newblock \emph{arXiv preprint arXiv:2402.03126}, 2024{\natexlab{a}}.

\bibitem[Attia and Koren(2024{\natexlab{b}})]{attia2024note}
A.~Attia and T.~Koren.
\newblock A note on high-probability analysis of algorithms with exponential, sub-gaussian, and general light tails.
\newblock \emph{arXiv preprint arXiv:2403.02873}, 2024{\natexlab{b}}.

\bibitem[Beck and Teboulle(2009)]{beck2009fast}
A.~Beck and M.~Teboulle.
\newblock A fast iterative shrinkage-thresholding algorithm for linear inverse problems.
\newblock \emph{SIAM journal on imaging sciences}, 2\penalty0 (1):\penalty0 183--202, 2009.

\bibitem[Bottou(2012)]{bottou2012stochastic}
L.~Bottou.
\newblock Stochastic gradient descent tricks.
\newblock In \emph{Neural networks: Tricks of the trade}, pages 421--436. Springer, 2012.

\bibitem[Carmon and Hinder(2022)]{carmon2022making}
Y.~Carmon and O.~Hinder.
\newblock Making sgd parameter-free.
\newblock In \emph{Conference on Learning Theory}, pages 2360--2389. PMLR, 2022.

\bibitem[Carmon and Hinder(2024)]{carmon2024price}
Y.~Carmon and O.~Hinder.
\newblock The price of adaptivity in stochastic convex optimization.
\newblock \emph{arXiv preprint arXiv:2402.10898}, 2024.

\bibitem[Chaudhuri et~al.(2009)Chaudhuri, Freund, and Hsu]{chaudhuri2009parameter}
K.~Chaudhuri, Y.~Freund, and D.~J. Hsu.
\newblock A parameter-free hedging algorithm.
\newblock \emph{Advances in neural information processing systems}, 22, 2009.

\bibitem[Chen et~al.(2020)Chen, Langford, and Orabona]{chen2022better}
K.~Chen, J.~Langford, and F.~Orabona.
\newblock Better parameter-free stochastic optimization with ode updates for coin-betting.
\newblock In \emph{AAAI Conference on Artificial Intelligence}, 2020.

\bibitem[Cutkosky(2019)]{cutkosky2019artificial}
A.~Cutkosky.
\newblock Artificial constraints and hints for unbounded online learning.
\newblock In \emph{Conference on Learning Theory}, pages 874--894. PMLR, 2019.

\bibitem[Cutkosky and Boahen(2017)]{cutkosky2017online}
A.~Cutkosky and K.~Boahen.
\newblock Online learning without prior information.
\newblock In \emph{Conference on learning theory}, pages 643--677. PMLR, 2017.

\bibitem[Cutkosky and Boahen(2016)]{cutkosky2016online}
A.~Cutkosky and K.~A. Boahen.
\newblock Online convex optimization with unconstrained domains and losses.
\newblock \emph{Advances in neural information processing systems}, 29, 2016.

\bibitem[Cutkosky and Orabona(2018)]{cutkosky2018black}
A.~Cutkosky and F.~Orabona.
\newblock Black-box reductions for parameter-free online learning in banach spaces.
\newblock In \emph{Conference On Learning Theory}, pages 1493--1529. PMLR, 2018.

\bibitem[Defazio and Mishchenko(2023)]{pmlr-v202-defazio23a}
A.~Defazio and K.~Mishchenko.
\newblock Learning-rate-free learning by d-adaptation.
\newblock In \emph{International Conference on Machine Learning}, 2023.

\bibitem[Duchi et~al.(2011)Duchi, Hazan, and Singer]{duchi2011adaptive}
J.~Duchi, E.~Hazan, and Y.~Singer.
\newblock Adaptive subgradient methods for online learning and stochastic optimization.
\newblock \emph{Journal of machine learning research}, 12\penalty0 (7), 2011.

\bibitem[Faw et~al.(2022)Faw, Tziotis, Caramanis, Mokhtari, Shakkottai, and Ward]{Faw2022ThePO}
M.~Faw, I.~Tziotis, C.~Caramanis, A.~Mokhtari, S.~Shakkottai, and R.~A. Ward.
\newblock The power of adaptivity in sgd: Self-tuning step sizes with unbounded gradients and affine variance.
\newblock In \emph{COLT}, 2022.

\bibitem[Ghadimi and Lan(2013)]{ghadimi2013stochastic}
S.~Ghadimi and G.~Lan.
\newblock Stochastic first-and zeroth-order methods for nonconvex stochastic programming.
\newblock \emph{SIAM Journal on Optimization}, 23\penalty0 (4):\penalty0 2341--2368, 2013.

\bibitem[Hazan and Kakade(2019)]{hazan2019revisiting}
E.~Hazan and S.~Kakade.
\newblock Revisiting the polyak step size.
\newblock \emph{arXiv preprint arXiv:1905.00313}, 2019.

\bibitem[Howard et~al.(2021)Howard, Ramdas, McAuliffe, and Sekhon]{howard2021time}
S.~R. Howard, A.~Ramdas, J.~McAuliffe, and J.~Sekhon.
\newblock Time-uniform, nonparametric, nonasymptotic confidence sequences.
\newblock \emph{The Annals of Statistics}, 49\penalty0 (2):\penalty0 1055--1080, 2021.

\bibitem[Ivgi et~al.(2023)Ivgi, Hinder, and Carmon]{Ivgi2023DoGIS}
M.~Ivgi, O.~Hinder, and Y.~Carmon.
\newblock Dog is sgd's best friend: A parameter-free dynamic step size schedule.
\newblock In \emph{International Conference on Machine Learning}, 2023.

\bibitem[Jun and Orabona(2019)]{jun2019parameter}
K.-S. Jun and F.~Orabona.
\newblock Parameter-free online convex optimization with sub-exponential noise.
\newblock In \emph{Conference on Learning Theory}, pages 1802--1823. PMLR, 2019.

\bibitem[Kavis et~al.(2019)Kavis, Levy, Bach, and Cevher]{kavis2019unixgrad}
A.~Kavis, K.~Y. Levy, F.~Bach, and V.~Cevher.
\newblock Unixgrad: A universal, adaptive algorithm with optimal guarantees for constrained optimization.
\newblock \emph{Advances in neural information processing systems}, 32, 2019.

\bibitem[Kavis et~al.(2022)Kavis, Levy, and Cevher]{kavis2022high}
A.~Kavis, K.~Y. Levy, and V.~Cevher.
\newblock High probability bounds for a class of nonconvex algorithms with adagrad stepsize.
\newblock In \emph{International Conference on Learning Representations}, 2022.

\bibitem[Khaled and Jin(2024)]{khaled2024tuning}
A.~Khaled and C.~Jin.
\newblock Tuning-free stochastic optimization.
\newblock \emph{arXiv preprint arXiv:2402.07793}, 2024.

\bibitem[Kingma and Ba(2015)]{kingma2014adam}
D.~P. Kingma and J.~Ba.
\newblock Adam: A method for stochastic optimization.
\newblock In \emph{International Conference on Learning Representations}, 2015.

\bibitem[Lan(2012)]{lan2012optimal}
G.~Lan.
\newblock An optimal method for stochastic composite optimization.
\newblock \emph{Mathematical Programming}, 133\penalty0 (1):\penalty0 365--397, 2012.

\bibitem[Li and Orabona(2019)]{li2019convergence}
X.~Li and F.~Orabona.
\newblock On the convergence of stochastic gradient descent with adaptive stepsizes.
\newblock In \emph{The 22nd International Conference on Artificial Intelligence and Statistics}, pages 983--992. PMLR, 2019.

\bibitem[Li and Orabona(2020)]{li2020high}
X.~Li and F.~Orabona.
\newblock A high probability analysis of adaptive sgd with momentum.
\newblock In \emph{Workshop on Beyond First Order Methods in ML Systems at ICML'20}, 2020.

\bibitem[Liu et~al.(2023)Liu, Nguyen, Nguyen, Ene, and Nguyen]{liu2023high}
Z.~Liu, T.~D. Nguyen, T.~H. Nguyen, A.~Ene, and H.~L. Nguyen.
\newblock High probability convergence of stochastic gradient methods.
\newblock \emph{arXiv preprint arXiv:2302.14843}, 2023.

\bibitem[Luo and Schapire(2015)]{luo2015achieving}
H.~Luo and R.~E. Schapire.
\newblock Achieving all with no parameters: Adanormalhedge.
\newblock In \emph{Conference on Learning Theory}, pages 1286--1304. PMLR, 2015.

\bibitem[McMahan and Orabona(2014)]{mcmahan2014unconstrained}
H.~B. McMahan and F.~Orabona.
\newblock Unconstrained online linear learning in hilbert spaces: Minimax algorithms and normal approximations.
\newblock In \emph{Conference on Learning Theory}, pages 1020--1039. PMLR, 2014.

\bibitem[Mhammedi and Koolen(2020)]{mhammedi2020lipschitz}
Z.~Mhammedi and W.~M. Koolen.
\newblock Lipschitz and comparator-norm adaptivity in online learning.
\newblock In \emph{Conference on Learning Theory}, pages 2858--2887. PMLR, 2020.

\bibitem[Mishchenko and Defazio(2023)]{mishchenko2023prodigy}
K.~Mishchenko and A.~Defazio.
\newblock Prodigy: An expeditiously adaptive parameter-free learner.
\newblock \emph{arXiv preprint arXiv:2306.06101}, 2023.

\bibitem[Nemirovskij and Yudin(1983)]{nemirovskij1983problem}
A.~S. Nemirovskij and D.~B. Yudin.
\newblock \emph{Problem complexity and method efficiency in optimization}.
\newblock Wiley-Interscience, 1983.

\bibitem[Nesterov(2015)]{nesterov2015universal}
Y.~Nesterov.
\newblock Universal gradient methods for convex optimization problems.
\newblock \emph{Mathematical Programming}, 152\penalty0 (1-2):\penalty0 381--404, 2015.

\bibitem[Orabona and P{\'a}l(2016)]{orabona2016coin}
F.~Orabona and D.~P{\'a}l.
\newblock Coin betting and parameter-free online learning.
\newblock \emph{Advances in Neural Information Processing Systems}, 29, 2016.

\bibitem[Orabona and P{\'a}l(2018)]{orabona2018scale}
F.~Orabona and D.~P{\'a}l.
\newblock Scale-free online learning.
\newblock \emph{Theoretical Computer Science}, 716:\penalty0 50--69, 2018.

\bibitem[Orabona and P{\'a}l(2021)]{orabona2021parameter}
F.~Orabona and D.~P{\'a}l.
\newblock Parameter-free stochastic optimization of variationally coherent functions.
\newblock \emph{arXiv preprint arXiv:2102.00236}, 2021.

\bibitem[Orabona and Tommasi(2017)]{orabona2017training}
F.~Orabona and T.~Tommasi.
\newblock Training deep networks without learning rates through coin betting.
\newblock \emph{Advances in Neural Information Processing Systems}, 30, 2017.

\bibitem[Reddi et~al.(2018)Reddi, Kale, and Kumar]{reddi2018convergence}
S.~J. Reddi, S.~Kale, and S.~Kumar.
\newblock On the convergence of adam and beyond.
\newblock In \emph{International Conference on Learning Representations}, 2018.

\bibitem[Robbins and Monro(1951)]{robbins1951stochastic}
H.~Robbins and S.~Monro.
\newblock A stochastic approximation method.
\newblock \emph{The annals of mathematical statistics}, pages 400--407, 1951.

\bibitem[Schaul et~al.(2013)Schaul, Zhang, and LeCun]{schaul2013no}
T.~Schaul, S.~Zhang, and Y.~LeCun.
\newblock No more pesky learning rates.
\newblock In \emph{International conference on machine learning}, pages 343--351. PMLR, 2013.

\bibitem[Streeter and McMahan(2012)]{Streeter2012NoRegretAF}
M.~J. Streeter and H.~B. McMahan.
\newblock No-regret algorithms for unconstrained online convex optimization.
\newblock In \emph{Neural Information Processing Systems}, 2012.

\bibitem[Tran et~al.(2019)]{tran2019convergence}
P.~T. Tran et~al.
\newblock On the convergence proof of amsgrad and a new version.
\newblock \emph{IEEE Access}, 7:\penalty0 61706--61716, 2019.

\bibitem[Ward et~al.(2019)Ward, Wu, and Bottou]{ward2019adagrad}
R.~Ward, X.~Wu, and L.~Bottou.
\newblock Adagrad stepsizes: Sharp convergence over nonconvex landscapes.
\newblock In \emph{International Conference on Machine Learning}, pages 6677--6686. PMLR, 2019.

\end{thebibliography}
\appendix

\section{Proofs of \texorpdfstring{\cref{sec:non-convex}}{Section 3}}\label{sec:proofs-non-convex}

\subsection{Proof of \texorpdfstring{\cref{lem:sgd-convergence}}{Lemma 1}}
In order to prove the lemma we use the following martingale concentration inequality due to \citet{li2020high}.
\begin{lemma}[Lemma 1 of \citet{li2020high}]
	\label{lem:sub_gaussian}
	Assume that $Z_1, Z_2, ..., Z_T$ is a martingale difference sequence with respect to $\xi_1, \xi_2, ..., \xi_T$ and $\E_t \left[\exp(Z_t^2/\sigma_t^2)\right] \leq \exp(1)$ for all $t$, where $\sigma_1,\ldots,\sigma_T$ is a sequence of random variables such that $\sigma_t$ is measurable with respect to $\xi_1, \xi_2, \dots, \xi_{t-1}$. 
	Then, for any fixed $\lambda > 0$ and $\delta \in (0,1)$, with probability at least $1-\delta$, it holds that
	\[
	\sum_{t=1}^T Z_t \leq \frac{3}{4} \lambda \sum_{t=1}^T \sigma_t^2 + \frac{1}{\lambda} \log \frac{1}{\delta}~.
	\]
\end{lemma}
\begin{proof}[Proof of \cref{lem:sgd-convergence}]%
    From smoothness,
    \begin{align*}
        f(w_{t+1})
        &\leq
        f(w_t)
        - \eta \nabla f(w_t) \cdot g_t
        + \frac{\smstar \eta^2}{2} \norm{g_t}^2
        .
    \end{align*}
    Summing for $t=1,\ldots,T$ and rearranging,
    \begin{align*}
        \sum_{t=1}^T \nabla f(w_t) \cdot g_t
        &\leq
        \frac{f(w_1)-f(w_{T+1})}{\eta}
        + \frac{\smstar \eta}{2} \sum_{t=1}^T \norm{g_t}^2
        .
    \end{align*}
    By \cref{lem:sub_gaussian} with $Z_t=\nabla f(w_t) \cdot (\nabla f(w_t)-g_t)$, where $\abs{Z_t} \leq \sigmastar \norm{\nabla f(w_t)}$, and $\lambda=\tfrac{1}{3 \sigmastar^2}$, with probability at least $1-\delta$,
    \begin{align*}
        \sum_{t=1}^T \nabla f(w_t) \cdot (\nabla f(w_t)-g_t)
        &
        \leq
        \frac{1}{4} \sum_{t=1}^T \norm{\nabla f(w_t)}^2
        + 3 \sigmastar^2 \log \tfrac{1}{\delta}
        .
    \end{align*}
    Summing both inequalities and rearranging,
    \begin{align*}
        \frac34 \sum_{t=1}^T \norm{\nabla f(w_t)}^2
        &\leq
        \frac{f(w_1)-f(w_{T+1})}{\eta}
        + \frac{\smstar \eta}{2} \sum_{t=1}^T \norm{g_t}^2
        + 3 \sigmastar^2 \log \tfrac{1}{\delta}
        .
    \end{align*}
    Since $\norm{g_t}^2 \leq 2 \norm{\nabla f(w_t)}^2 + 2 \sigmastar^2$,
    \begin{align*}
        \frac34 \sum_{t=1}^T \norm{\nabla f(w_t)}^2
        &\leq
        \frac{f(w_1)-f(w_{T+1})}{\eta}
        + \smstar \eta \sum_{t=1}^T \norm{\nabla f(w_t)}^2
        + \smstar \sigmastar^2 T \eta
        + 3 \sigmastar^2 \log \tfrac{1}{\delta}
        .
    \end{align*}
    As $\eta \leq \frac{1}{2 \smstar}$, substituting and rearranging,
    \begin{align*}
        \sum_{t=1}^T \norm{\nabla f(w_t)}^2
        &\leq
        \frac{4\brk{f(w_1)-f(w_{T+1})}}{\eta}
        + 4 \smstar \sigmastar^2 T \eta
        + 12 \sigmastar^2 \log \tfrac{1}{\delta}
        .
    \end{align*}
    Dividing by $\frac1T$ and replacing $f(w_{T+1}) \geq f^\star$ we conclude the proof.
\end{proof}

\subsection{Proof of \texorpdfstring{\cref{lem:minimal-norm-candidate}}{Lemma 2}}

In order to prove \cref{lem:minimal-norm-candidate} we use the following martingale lemma.
\begin{lemma}[Lemma 2.3 of \citet{ghadimi2013stochastic}]\label{lem:mds-sub-gaussian-bound}
    Let $Z_1,\ldots,Z_n \in \reals^d$ be a martingale difference sequence with respect to $\xi_1,\ldots,\xi_n$.
    Assuming $\E[\exp(\norm{Z_i}^2/\sigma^2) \mid \xi_1,\ldots,\xi_{i-1}] \leq \exp(1)$, for any $\lambda > 0$,
    \begin{align*}
        \Pr\brk*{\norm*{\sum_{i=1}^n Z_i} \geq \sqrt{2}(1+\lambda)\sigma \sqrt{n}}
        &\leq
        \exp(-\lambda^2/3)
        .
    \end{align*}
\end{lemma}
\begin{proof}[Proof of \cref{lem:minimal-norm-candidate}]
Denote $\bar{g}(w) \eqdef \frac{1}{T} \sum_{t=1}^T \goracle_t(w)$. Let $k = \argmin_{n \in [N]} \norm{\nabla f(w_n)}$. Using $\norm{a+b}^2 \leq 2 \norm{a}^2 + 2 \norm{b}^2$ and the minimality of $\wout$,
\begin{align*}
    \norm{\nabla f(\wout)}^2
    &\leq
    2\norm{\bar{g}(\wout)}^2
    + 2 \norm{g(\wout)-\nabla f(\wout)}^2
    \leq
    2\norm{\bar{g}(w_k)}^2
    + 2 \norm{\bar{g}(\wout)-\nabla f(\wout)}^2
    \\
    &\leq
    4 \norm{\nabla f(w_k)}^2
    + 4 \norm{\bar{g}(w_k) - \nabla f(w_k)}^2
    + 2 \norm{\bar{g}(\wout)-\nabla f(\wout)}^2
    .
\end{align*}
Let $n \in [N]$.
By \cref{lem:mds-sub-gaussian-bound} (note that the bounded noise assumption satisfy the sub-Gaussian condition), for any $\lambda > 0$,
\begin{align*}
    \Pr\brk2{\norm{\bar{g}(w_n)-\nabla f(w_n)}^2 \geq \frac{2(1+\lambda)^2 \sigma^2}{T}}
    &=
    \Pr\brk2{\norm{\sum_{t=1}^T (\goracle_t(w_n)- \nabla f(w_n))} \geq \sqrt{2}(1+\lambda) \sigma \sqrt{T}}
    \leq
    e^{-\ifrac{\lambda^2}{3}}
    .
\end{align*}
Thus, under a union bound,
\begin{align*}
    \Pr\brk*{\exists n \in [N], \norm{\bar{g}(w_n)-\nabla f(w_n)}^2 \geq \frac{2(1+\lambda)^2 \sigma^2}{T}}
    &\leq
    S \exp(-\lambda^2/3)
    .
\end{align*}
Setting $\lambda=\sqrt{3 \log \frac{N}{\delta}}$, with probability at least $1-\delta$, for all $n \in [N]$,
\begin{align*}
    \norm{\bar{g}(w_n)-\nabla f(w_n)}^2
    &
    \leq \frac{4(1+\lambda^2)\sigma^2}{T}
    \leq \frac{4(1+3 \log \tfrac{N}{\delta})\sigma^2}{T}
    .
\end{align*}
Thus, with probability at least $1-\delta$.
\begin{align*}
    \norm{\nabla f(\wout)}^2
    &
    \leq 4 \norm{\nabla f(w_k)}^2
    + 4 \norm{\bar{g}(w_k) - \nabla f(w_k)}^2
    + 2 \norm{\bar{g}(\wout)-\nabla f(\wout)}^2
    \\&
    \leq 4 \norm{\nabla f(w_k)}^2
    + \frac{24 (1+3 \log \tfrac{N}{\delta}) \sigma^2}{T}
    .
    \qedhere
\end{align*}
\end{proof}

\section{Proofs of \texorpdfstring{\cref{sec:convex}}{Section 4}}\label{sec:convex-proofs}

\subsection{Proof of \texorpdfstring{\cref{lem:convex-sgd-expectation}}{Lemma 3}}

\begin{proof}[\unskip\nopunct]%
    By the update step of SGD,
    \begin{align*}
        \norm{w_{t+1}-w^\star}^2
        &= \norm{w_t-w^\star}^2
        - 2 \eta g_t \cdot (w_t-w^\star)
        + \eta^2 \norm{g_t}^2
        .
    \end{align*}
    Summing for $t=1,\ldots,T$ and rearranging,
    \begin{align*}
        \sum_{t=1}^T g_t \cdot (w_t-w^\star)
        &=
        \frac{\norm{w_1-w^\star}^2-\norm{w_{T+1}-w^\star}^2}{2 \eta} + \frac{\eta}{2} \sum_{t=1}^T \norm{g_t}^2
        .
    \end{align*}
    Note that by the Lipschitz and noise assumptions, $\E[\norm{g_t}^2]=\E[\norm{g_t-\nabla f(w_t)}^2 - 2 \nabla f(w_t) \cdot (\nabla f(w_t)-g_t) + \norm{\nabla f(w_t)}^2] \leq \lipstar^2 + \sigmastar^2$. Taking expectation,
    \begin{align*}
        \E\brk[s]*{\sum_{t=1}^T \nabla f(w_t) \cdot (w_t-w^\star)}
        &\leq
        \frac{\norm{w_1-w^\star}^2}{2 \eta} + \frac{\eta (\lipstar^2 + \sigmastar^2) T}{2}
        .
    \end{align*}
    We conclude by an application of convexity and Jensen's inequality.
\end{proof}
\subsection{Proof of \texorpdfstring{\cref{lem:minimal-function-candidate}}{Lemma 4}}
\begin{proof}[\unskip\nopunct]%
    Let $k = \argmin_{n \in [N]} f(w_n)$. By the minimality of $\wout$,
    \begin{align*}
        f(\wout)
        &=
        \frac{1}{T} \sum_{t=1}^T \foracle_t(\wout)
        + \brk*{f(\wout) - \frac{1}{T} \sum_{t=1}^T \foracle_t(\wout)}
        \leq
        \frac{1}{T} \sum_{t=1}^T \foracle_t(w_k)
        + \brk*{f(\wout) - \frac{1}{T} \sum_{t=1}^T \foracle_t(\wout)}
        \\
        &=
        f(w_k) + \brk*{\frac{1}{T} \sum_{t=1}^T \foracle_t(w_k)-f(w_k)}
        + \brk*{f(\wout) - \frac{1}{T} \sum_{t=1}^T \foracle_t(\wout)}
        .
    \end{align*}
    Let $n \in [N]$.
    By Hoeffding's inequality, for any $\epsilon > 0$,
    \begin{align*}
        \Pr\brk*{\abs*{f(w_n) - \frac{1}{T} \sum_{t=1}^T \foracle_t(w_n)} \geq \epsilon} \leq 2 \exp\brk*{-\frac{\epsilon^2 T}{2 \fsigma^2}}.
    \end{align*}
    Thus, under a union bound,
    \begin{align*}
        \Pr\brk*{\exists n \in [N] : \abs*{f(w_n) - \frac{1}{T} \sum_{t=1}^T \foracle_t(w_n)} \geq \epsilon} \leq 2 N \exp\brk*{-\frac{\epsilon^2 T}{2 \fsigma^2}}.
    \end{align*}
    Setting $\epsilon=\sqrt{\ifrac{2 \fsigma^2 \log \tfrac{2 N}{\delta}}{T}}$, with probability at least $1-\delta$, for all $n \in [N]$,
    \begin{align*}
        \abs*{f(w_n) - \frac{1}{T} \sum_{t=1}^T \foracle_t(w_n)} \leq \sqrt{\frac{2 \fsigma^2 \log \tfrac{2 N}{\delta}}{T}}
        .
    \end{align*}
    Thus, with probability at least $1-\delta$,
    \begin{align*}
        f(\wout)
        &\leq
        f(w_k) + \sqrt{\frac{8 \fsigma^2 \log \tfrac{2 N}{\delta}}{T}}
        .
        \qedhere
    \end{align*}
\end{proof}

\subsection{Proof of \texorpdfstring{\cref{lem:padasgd-convergence}}{Lemma 5}}

In order to prove the lemma we use the following standard inequality.

\begin{lemma}\label{lem:sum-one-over-sum-sqrt}
    Let $a_1,\ldots,a_n \geq 0$. Then
    \begin{align*}
        \sum_{i=1}^n \frac{a_i}{\sqrt{\sum_{j=1}^i a_j}} \leq 2 \sqrt{\sum_{i=1}^n a_i}.
        \tag{treating $\tfrac{0}{0} \eqdef 0$}
    \end{align*}
\end{lemma}

\begin{proof}%
    Using the identity $a^2-b^2=(a-b)(a+b)$,
    \begin{align*}
        \sum_{i=1}^n \frac{a_i}{\sqrt{\sum_{j=1}^i a_j}}
        &=
        \sum_{i=1}^n \frac{\brk*{\sqrt{\sum_{j=1}^i a_j}-\sqrt{\sum_{j=1}^{i-1} a_j}}\brk*{\sqrt{\sum_{j=1}^i a_j}+\sqrt{\sum_{j=1}^{i-1} a_j}}}{\sqrt{\sum_{j=1}^i a_j}}
        \\
        &
        \leq
        2 \sum_{i=1}^n \brk*{\sqrt{\sum_{j=1}^i a_j}-\sqrt{\sum_{j=1}^{i-1} a_j}}
        = 2 \sqrt{\sum_{i=1}^n a_i}
        .
        \qedhere
    \end{align*}
\end{proof}

We proceed to prove the lemma.

\begin{proof}[Proof of \cref{lem:padasgd-convergence}]
    As $\event_1$ holds, $\max_{t \leq T} \norm{w_{t+1}-w_1} < \diam$ and $w_{t+1}=w_t-\eta_t g_t$ for all $t \in [T]$ (if a projection is performed the norm will equal $\diam$). Thus,
    \begin{align*}
        \norm{w_{t+1}-w^\star}^2
        &=
        \norm{w_t-w^\star}^2
        - 2 \eta_t g_t \cdot (w_t-w^\star)
        + \eta_t^2 \norm{g_t}^2
        .
    \end{align*}
    Rearranging and summing for $t=1,\ldots,T$,
    \begin{align*}
        \sum_{t=1}^T g_t \cdot (w_t-w^\star)
        &=
        \frac{\norm{w_1-w^\star}^2}{2 \eta_1}
        - \frac{\norm{w_{T+1}-w^\star}^2}{2 \eta_T}
        + \frac12 \sum_{t=2}^T \brk*{\frac{1}{\eta_t}-\frac{1}{\eta_{t-1}}} \norm{w_t-w^\star}^2
        + \frac12 \sum_{t=1}^T \eta_t \norm{g_t}^2
        .
    \end{align*}
    Let $t' \in \argmax_{t \leq T+1} \norm{w_t-w^\star}$.
    Hence,
    \begin{align*}
        \sum_{t=1}^T g_t \cdot (w_t-w^\star)
        &\leq
        \frac{\norm{w_1-w^\star}^2}{2 \eta_1}
        - \frac{\norm{w_{T+1}-w^\star}^2}{2 \eta_T}
        + \frac{\norm{w_{t'}-w^\star}^2}{2} \sum_{t=2}^T \brk*{\frac{1}{\eta_t}-\frac{1}{\eta_{t-1}}}
        + \frac12 \sum_{t=1}^T \eta_t \norm{g_t}^2
        \\
        &=
        \frac{\norm{w_1-w^\star}^2}{2 \eta_1}
        - \frac{\norm{w_{T+1}-w^\star}^2}{2 \eta_T}
        + \frac{\norm{w_{t'}-w^\star}^2}{2} \brk*{\frac{1}{\eta_T}-\frac{1}{\eta_{1}}}
        + \frac12 \sum_{t=1}^T \eta_t \norm{g_t}^2
        \\
        &\leq
        \frac{\norm{w_{t'}-w^\star}^2-\norm{w_{T+1}-w^\star}^2}{2 \eta_T}
        + \frac12 \sum_{t=1}^T \eta_t \norm{g_t}^2
        .
        \tag{$\norm{w_1-w^\star} \leq \norm{w_{t'}-w^\star}$}
    \end{align*}
    Note that by the triangle inequality,
    \begin{align*}
        \norm{w_{t'}-w^\star}^2 &- \norm{w_{T+1}-w^\star}^2
        =
        \brk{\norm{w_{t'}-w^\star}-\norm{w_{T+1}-w^\star}}\brk{\norm{w_{t'}-w^\star}+\norm{w_{T+1}-w^\star}}
        \\
        &\leq
        2 \norm{w_{t'}-w_{T+1}} \norm{w_{t'}-w^\star}
        \tag{triangle inequality and def of $t'$}
        \\
        &\leq
        2 \norm{w_{t'}-w_{T+1}} (\diam+\norm{w_1-w^\star})
        \tag{$\norm{w_{t'}-w^\star} \leq \norm{w_{t'}-w_1}+\norm{w_1-w^\star}$}
        \\
        &\leq
        4 \diam \brk{\diam+\norm{w_1-w^\star}}
        .
        \tag{$\norm{w_{t'}-w_{T+1}} \leq \norm{w_{t'}-w_1}+\norm{w_1-w_{T+1}} \leq 2 \diam$}
    \end{align*}
    Plugging the above inequality,
    \begin{align*}
        \sum_{t=1}^T g_t \cdot (w_t-w^\star)
        &\leq
        \frac{2 \diam \brk{\diam+\norm{w_1-w^\star}}}{\eta_T}
        + \frac12 \sum_{t=1}^T \eta_t \norm{g_t}^2
        .
    \end{align*}
    By \cref{lem:sum-one-over-sum-sqrt} with $a_i = \norm{g_i}^2$,
    \begin{align*}
        \sum_{t=1}^T \eta_t \norm{g_t}^2
        &\leq
        2 \initeta \diam \sqrt{\sum_{t=1}^T \norm{g_t}^2}.
    \end{align*}
    Combining the two inequalities with the inequality of $\event_2$,
    \begin{align*}
        \sum_{t=1}^T & \nabla f(w_t) \cdot (w_t-w^\star)
        \leq
        \frac{2 \diam \brk{\diam+\norm{w_1-w^\star}}}{\eta_T}
        + \initeta \diam \sqrt{\sum_{t=1}^T \norm{g_t}^2}
        \\&
        + 4 \brk{\diam+\norm{w_1-w^\star}} \sqrt{\theta_{T,\delta} \sum_{t=1}^T \norm{g_t}^2 + 4 \sigmaball^2 \theta_{t,\delta}^2}
        \\
        &\leq
        \brk*{\frac{2(\diam+\norm{w_1-w^\star})}{\initeta}+\initeta \diam +4 \brk{\diam + \norm{w_1-w^\star}}\sqrt{\theta_{T,\delta}}}\sqrt{\sum_{t=1}^T \norm{g_t}^2}
        \\&
        +\frac{2 \brk{\diam + \norm{w_1-w^\star}}\initg}{\initeta}
        + 8 \theta_{T,\delta} \brk{\diam + \norm{w_1-w^\star}} \sigmaball
        \tag{$\sqrt{a+b} \leq \sqrt{a}+\sqrt{b}$}
        \\
        &=
        \brk*{\brk{\norm{w_1-w^\star}+\diam}\brk{\tfrac{2}{\initeta}+4 \sqrt{\theta_{T,\delta}}} + \initeta \diam}\sqrt{\sum_{t=1}^T \norm{g_t}^2}
        + \brk*{\tfrac{2 \initg}{\initeta}
        + 8 \theta_{T,\delta} \sigmaball}\brk{\norm{w_1-w^\star}+\diam}
        .
    \end{align*}
    Dividing by $T$, we conclude by applying $\nabla f(w_t) \cdot (w_t-w^\star) \geq f(w_t)-f(w^\star)$ due to convexity.
\end{proof}

\subsection{Proof of \texorpdfstring{\cref{lem:padasgd-iterate-bound}}{Lemma 6}}

In order to prove lemma \cref{lem:padasgd-iterate-bound} we use the following lemmas. Their proofs will be presented subsequently.
The first lemma bounds the errors that arise from potential correlation between $\eta_t$ and $g_t$. To that end we use the following ``decorrelated'' stepsize \citep{attia2023sgd},
\begin{align*}
    \tilde{\eta}_t
    &\eqdef
    \frac{\initeta \diam}{\sqrt{\initg^2+\norm{\nabla f(w_t)}^2 + \sum_{s=1}^{t-1} \norm{g_s}^2}}
    .
\end{align*}
In the noiseless case we could simply replace the lemma with $\nabla f(w_s) \cdot (w_s-w^\star_\diam) \geq 0$ due to convexity.
\begin{lemma}\label{lem:padasgd-noisy-convex-inequality}
    With probability at least $1-\delta$, for any $t \in [T]$,
    \begin{align*}
        -\sum_{s=1}^t \eta_s g_s \cdot (w_s-w^\star_\diam)
        &\leq
        \frac{1}{\initeta} \sum_{s=1}^t \eta_s^2 \norm{g_s}^2
        + \frac{1}{\initeta} \sum_{s=1}^t \tilde{\eta}_s^2 \norm{g_s-\nabla f(w_s)}^2
        \\
        &+
        8 \diam \sqrt{\theta_{t,\delta} \sum_{s=1}^t \tilde{\eta}_s^2 \norm{\nabla f(w_s)-g_s}^2 + \frac{\initeta^2 \diam^2 \sigmaball^2 \theta_{t,\delta}^2}{\initg^2}}
        ,
    \end{align*}
    where $\theta_{t,\delta}=\log \tfrac{60 \log (6t)}{\delta}$.
\end{lemma}
Following is a standard summation lemma of methods with AdaGrad-like stepsizes.
\begin{lemma}\label{lem:sum_eta_squared}
    For any $t \in [T]$,
    \begin{align*}
        \sum_{s=1}^t \eta_s^2 \norm{g_s}^2
        &\leq
        \initeta^2 \diam^2 \log \brk*{1+\frac{2 \lipball^2 T + 2 \sigmaball^2 T}{\initg^2}}
        .
    \end{align*}
\end{lemma}
The next lemma is similar to \cref{lem:sum_eta_squared}, while using decorrelated stepsizes. Here, knowledge of the noise bound is required to achieve a logarithmic summation.
\begin{lemma}\label{lem:sum_tilde_eta_squared}
    Assuming $\initg^2 \geq 5 \sigmaball^2 \log \tfrac{T}{\delta}$, with probability at least $1-\delta$, for any $t \in [T]$,
    \begin{align*}
        \sum_{s=1}^t \tilde{\eta}_s^2 \norm{g_s-\nabla f(w_s)}^2
        &\leq
        \initeta^2 \diam^2 \log (1+T)
        .
    \end{align*}
\end{lemma}
\begin{proof}[Proof of \cref{lem:padasgd-iterate-bound}]
    As $w_{s+1}$ is the projection of $w_s-\eta_s g_s$ onto the convex set containing $w^\star_\diam$,
    \begin{align*}
        \norm{w_{s+1}-w^\star_\diam}^2
        \leq
        \norm{w_s-w^\star_\diam}^2
        - 2 \eta_s g_s \cdot (w_s-w^\star_\diam)
        + \eta_s^2 \norm{g_s}^2
        .
    \end{align*}
    Summing for $s=1,\ldots,t$,
    \begin{align*}
        \norm{w_{t+1}-w^\star_\diam}^2
        \leq
        \norm{w_1-w^\star_\diam}^2
        - 2 \sum_{s=1}^t \eta_s g_s \cdot (w_s-w^\star_\diam)
        + \sum_{s=1}^t \eta_s^2 \norm{g_s}^2
        .
    \end{align*}
    Using lemma \cref{lem:padasgd-noisy-convex-inequality}, with probability at least $1-\delta$, for any $t \in [T]$,
    \begin{align*}
        \norm{w_{t+1}-w^\star_\diam}^2
        &\leq
        \norm{w_1-w^\star_\diam}^2
        + 16 \diam \sqrt{\theta_{t,\delta} \sum_{s=1}^t \tilde{\eta}_s^2 \norm{\nabla f(w_s)-g_s}^2 + \frac{\initeta^2 \diam^2 \sigmaball^2 \theta_{t,\delta}^2}{\initg^2}}
        \\&
        + \frac{2}{\initeta} \sum_{s=1}^t \tilde{\eta}_s^2 \norm{g_s-\nabla f(w_s)}^2
        + \brk*{1+\frac{2}{\initeta}}\sum_{s=1}^t \eta_s^2 \norm{g_s}^2
        .
    \end{align*}
    Applying \cref{lem:sum_eta_squared,lem:sum_tilde_eta_squared} under a union bound, with probability at least $1-2 \delta$, for any $t \in [T]$,
    \begin{align*}
        \norm{w_{t+1}-w^\star_\diam}^2
        &\leq
        \norm{w_1-w^\star_\diam}^2
        + 16 \initeta \diam^2 \sqrt{\theta_{t,\delta} \log \brk{1+T} + \frac{\sigmaball^2 \theta_{t,\delta}^2}{\initg^2}}
        \\&
        + 2 \initeta \diam^2 \log \brk{1+T}
        + \initeta \brk{\initeta+2} \diam^2 \log\brk*{1+\frac{2 \lipball^2 T + 2 \sigmaball^2 T}{\initg^2}}
        \\
        &\leq
        \norm{w_1-w^\star_\diam}^2
        + \frac{\diam^2}{8}
        + \frac{\diam^2}{16}
        + \frac{\diam^2}{16}
        \tag{Defs of $\initeta,\initg$ at \cref{alg:convex-lipschitz-new} and $\initeta \leq 1$}
        \\
        &=
        \norm{w_1-w^\star_\diam}^2
        + \frac14 \diam^2
        .
    \end{align*}
    Thus, using the triangle inequality and the fact that $\norm{w_1-w^\star} \geq \norm{w_1-w^\star_\diam}$,
    \begin{align*}
        \norm{w_{t+1}-w_1}
        &\leq
        \norm{w_{t+1}-w^\star_\diam}
        + \norm{w^\star_\diam - w_1}
        \leq
        2 \norm{w_1-w^\star_\diam} + \frac12 \diam
        \leq
        2 \norm{w_1-w^\star} + \frac12 \diam
        .
    \end{align*}
    In addition,
    \begin{align*}
        \norm{w_{t+1}-w^\star}
        &\leq
        \norm{w_{t+1}-w_1}
        + \norm{w_1-w^\star}
        \leq
        3 \norm{w_1-w^\star}
        + \frac12 \diam
        .
        \qedhere
    \end{align*}
\end{proof}

\subsection{Proof of \texorpdfstring{\cref{lem:padasgd-noisy-convex-inequality}}{Lemma 11}}
Using $\tilde{\eta}_t$ we can decompose the term $-\eta_s g_t \cdot (w_s-w^\star_\diam)$,
\begin{align*}
    - \eta_s g_s \cdot (w_s-w^\star_\diam)
    &=
    - \tilde{\eta}_s g_s \cdot (w_s-w^\star_\diam)
    + \brk{\tilde{\eta}_s-\eta_s} g_s \cdot (w_s-w^\star_\diam)
    \\&
    =
    - \tilde{\eta}_s \nabla f(w_s) \cdot (w_s-w^\star_\diam)
    + \tilde{\eta}_s \brk{\nabla f(w_s)-g_s} \cdot (w_s-w^\star_\diam)
    + \brk{\tilde{\eta}_s-\eta_s} g_s \cdot (w_s-w^\star_\diam)
    \\&
    \leq
    \tilde{\eta}_s \brk{\nabla f(w_s)-g_s} \cdot (w_s-w^\star_\diam)
    + \brk{\tilde{\eta}_s-\eta_s} g_s \cdot (w_s-w^\star_\diam)
    ,
\end{align*}
where the last inequality follows from convexity.
Thus,
\begin{align}\label{eq:eta-decomposition}
    - \sum_{s=1}^t \eta_s g_s \cdot (w_s-w^\star_\diam)
    &\leq
    \sum_{s=1}^t \tilde{\eta}_s \brk{\nabla f(w_s)-g_s} \cdot (w_s-w^\star_\diam)
    + \sum_{s=1}^t \brk{\tilde{\eta}_s-\eta_s} g_s \cdot (w_s-w^\star_\diam)
    .
\end{align}
The next two lemmas bounds the two sums and concludes the proof of \cref{lem:padasgd-noisy-convex-inequality}. Their proofs follows.
\begin{lemma}\label{lem:padasgd-martingale}
    With probability at least $1-\delta$, for any $t \in [T]$,
    \begin{align*}
        \sum_{s=1}^t \tilde{\eta}_s \brk{\nabla f(w_s)-g_s} \cdot (w_s-w^\star_\diam)
        &\leq
        8 \diam \sqrt{\theta_{t,\delta} \sum_{s=1}^t \tilde{\eta}_s^2 \norm{\nabla f(w_s)-g_s}^2 + \frac{\initeta^2 \diam^2 \sigmaball^2 \theta_{t,\delta}^2}{\initg^2}}
        .
    \end{align*}
\end{lemma}
\begin{lemma}\label{lem:padasgd-correlation-error}
    For any $t \in [T]$,
    \begin{align*}
        \sum_{s=1}^t \brk{\tilde{\eta}_s-\eta_s} g_s \cdot (w_s-w^\star_\diam)
        &\leq
        \frac{1}{\initeta} \sum_{s=1}^t \eta_s^2 \norm{g_s}^2
        + \frac{1}{\initeta} \sum_{s=1}^t \tilde{\eta}_s^2 \norm{g_s-\nabla f(w_s)}^2
        .
    \end{align*}
\end{lemma}
Combining both into \cref{eq:eta-decomposition} yields the desired inequality.%
\qed

\subsection{Proof of \texorpdfstring{\cref{lem:padasgd-martingale}}{Lemma 14}}
While $\norm{\nabla f(w_s)-g_s} \leq \sigmaball$, this bound is too crude for the required martingale analysis. To this end we use the following empirical Bernstein inequality due to \citet{howard2021time} in order to exploit the relationship between $\tilde{\eta}_t$ and $\nabla f(w_s)-g_s$.
\begin{lemma}[Corollary 2 of \citet{Ivgi2023DoGIS}]\label{lem:emp-bernstein}
    Let $c>0$ and $X_t$ be a martingale difference sequence adapted to $\cF_t$ such that $\abs{X_t} \leq c$ with probability $1$ for all $t$. Then, for all $\delta \in (0,1)$, and $\hat{X}_t \in \cF_{t-1}$ such that $\abs{\hat{X}_t} \leq c$ with probability $1$,
    \begin{align*}
        \Pr\brk*{\exists t \in [T] : \abs*{\sum_{s=1}^t X_s} \geq 4 \sqrt{\theta_{t,\delta} \sum_{s=1}^t \brk*{X_s-\hat{X}_s}^2 + c^2 \theta_{t,\delta}^2}}
        \leq
        \delta
        ,
    \end{align*}
    where $\theta_{t,\delta}=\log \tfrac{60 \log (6t)}{\delta}$.
\end{lemma}
\begin{proof}[\unskip\nopunct]
Invoking \cref{lem:emp-bernstein} with
\begin{align*}
    X_s &= \tilde{\eta}_s \brk{\nabla f(w_s)-g_s} \cdot (w_s-w^\star_\diam),
    \qquad
    \hat{X}_s = 0
    \qquad\text{and}\qquad
    c = \frac{2 \initeta \diam^2 \sigmaball}{\initg},
\end{align*}
(note the use of $\diam \leq 8 \norm{w_1-w^\star}$ to bound $\norm{\nabla f(w_s)-g_s} \leq \sigmaball$) with probability at least $1-\delta$, for any $t \in [T]$,
\begin{align*}
    \abs*{\sum_{s=1}^t \tilde{\eta}_s \brk{\nabla f(w_s)-g_s} \cdot (w_s-w^\star_\diam)}
    \leq
    4 \sqrt{\theta_{t,\delta} \sum_{s=1}^t \brk*{\tilde{\eta}_s \brk{\nabla f(w_s)-g_s} \cdot (w_s-w^\star_\diam)}^2 + \frac{4 \initeta^2 \diam^4 \sigmaball^2 \theta_{t,\delta}^2}{\initg^2}}
    .
\end{align*}
By Cauchy-Schwarz inequality and the triangle inequality ($\norm{w_s-w^\star_\diam} \leq \norm{w_s-w_1}+\norm{w_1-w^\star_\diam} \leq 2 \diam$),
\begin{align*}
    \brk*{\tilde{\eta}_s \brk{\nabla f(w_s)-g_s} \cdot (w_s-w^\star_\diam)}^2
    &\leq
    \tilde{\eta}_s^2 \norm{\nabla f(w_s)-g_s}^2 \norm{w_s-w^\star_\diam}^2
    \leq
    4 \tilde{\eta}_s^2 \norm{\nabla f(w_s)-g_s}^2 \diam^2
    .
\end{align*}
Hence, with probability at least $1-\delta$, for any $t \in [T]$,
\begin{align*}
    \abs*{\sum_{s=1}^t \tilde{\eta}_s \brk{\nabla f(w_s)-g_s} \cdot (w_s-w^\star_\diam)}
    &\leq
    8 \diam \sqrt{\theta_{t,\delta} \sum_{s=1}^t \tilde{\eta}_s^2 \norm{\nabla f(w_s)-g_s}^2 + \frac{\initeta^2 \diam^2 \sigmaball^2 \theta_{t,\delta}^2}{\initg^2}}
    .
    \qedhere
\end{align*}
\end{proof}
\subsection{Proof of \texorpdfstring{\cref{lem:padasgd-correlation-error}}{Lemma 15}}
\begin{proof}[\unskip\nopunct]%
    Let $G_t = \sqrt{\sum_{s=1}^t \norm{g_s}^2}$. Thus,
    \begin{align*}
        \abs*{\tilde{\eta}_s-\eta_s}
        &=
        \frac{\initeta \diam \abs{\norm{g_s}^2-\norm{\nabla f(w_s)}^2}}{\sqrt{\sigmaball^2 + G_s^2} \sqrt{\sigmaball^2 + \norm{\nabla f(w_s)}^2 + G_{s-1}^2}\brk{\sqrt{\sigmaball^2 + G_s^2}+ \sqrt{\sigmaball^2 + \norm{\nabla f(w_s)}^2 + G_{s-1}^2}}}
        .
    \end{align*}
    By the triangle inequality,
    \begin{align*}
        \frac{\abs{\norm{g_s}^2-\norm{\nabla f(w_s)}^2}}{\sqrt{\sigmaball^2 + G_s^2}+ \sqrt{\sigmaball^2 + \norm{\nabla f(w_s)}^2 + G_{s-1}^2}}
        &\leq
        \frac{\norm{g_s-\nabla f(w_s)} \brk{\norm{g_s}+\norm{\nabla f(w_s)}}}{\sqrt{\sigmaball^2 + G_s^2}+ \sqrt{\sigmaball^2 + \norm{\nabla f(w_s)}^2 + G_{s-1}^2}}
        \\
        &
        \leq
        \norm{g_s-\nabla f(w_s)}
        .
    \end{align*}
    Hence, by Cauchy-Schwarz inequality,
    \begin{align*}
        \sum_{s=1}^t \brk{\tilde{\eta}_s-\eta_s} g_s \cdot (w_s-w^\star_\diam)
        &\leq
        \initeta \diam \sum_{s=1}^t \frac{\norm{g_s-\nabla f(w_s)} \norm{g_s} \norm{w_s-w^\star_\diam}}{\sqrt{\sigmaball^2 + G_s^2} \sqrt{\sigmaball^2 + \norm{\nabla f(w_s)}^2 + G_{s-1}^2}}
        \\
        &\leq
        2 \initeta \diam^2 \sum_{s=1}^t \frac{\norm{g_s-\nabla f(w_s)} \norm{g_s}}{\sqrt{\sigmaball^2 + G_s^2} \sqrt{\sigmaball^2 + \norm{\nabla f(w_s)}^2 + G_{s-1}^2}}
        \\
        &
        =
        \frac{2}{\initeta} \sum_{s=1}^t \eta_s \tilde{\eta}_s \norm{g_s-\nabla f(w_s)} \norm{g_s}
        ,
    \end{align*}
    where the second inequality follows by $\norm{w_s-w^\star_\diam} \leq 2 \diam$.
    Using the identity $2 a b \leq a^2 + b^2$,
    \begin{align*}
        \sum_{s=1}^t \brk{\tilde{\eta}_s-\eta_s} g_s \cdot (w_s-w^\star_\diam)
        &\leq
        \frac{1}{\initeta} \sum_{s=1}^t \eta_s^2 \norm{g_s}^2
        + \frac{1}{\initeta} \sum_{s=1}^t \tilde{\eta}_s^2 \norm{g_s-\nabla f(w_s)}^2
        .
        \qedhere
    \end{align*}
\end{proof}
\subsection{Proof of \texorpdfstring{\cref{lem:sum_eta_squared,lem:sum_tilde_eta_squared}}{Lemmas 12 and 13}}
In order to prove the lemmas we need the following standard lemma used in the analysis of AdaGrad-like methods.
\begin{lemma}\label{lem:sum-one-over-sum}
    Let $a_0 > 0$ and $a_1,a_2,\ldots,a_n \geq 0$. Then
    \begin{align*}
        \sum_{i=1}^n \frac{a_i}{\sum_{j=0}^i a_j} \leq \ln \frac{\sum_{i=0}^n a_i}{a_0} \leq \log \frac{\sum_{i=0}^n a_i}{a_0}.
    \end{align*}
\end{lemma}
\begin{proof}
    Using the inequality $1-x \leq \ln\brk{\ifrac{1}{x}}$ for $x > 0$,
    \begin{align*}
        \sum_{i=1}^n \frac{a_i}{\sum_{j=0}^i a_j}
        =
        \sum_{i=1}^n \brk*{1-\frac{\sum_{j=0}^{i-1} a_j}{\sum_{j=0}^i a_j}}
        \leq
        \sum_{i=1}^n \ln \brk*{\frac{\sum_{j=0}^i a_j}{\sum_{j=0}^{i-1} a_j}}
        =
        \ln \brk*{\frac{\sum_{j=0}^n a_j}{a_0}}
        \leq
        \log \brk*{\frac{\sum_{j=0}^n a_j}{a_0}}
        ,
    \end{align*}
    where the last inequality holds since $\ln \brk{\ifrac{\sum_{j=0}^n a_j}{a_0}} \geq 0$.
\end{proof}

\begin{proof}[Proof of \cref{lem:sum_eta_squared}]
    Using \cref{lem:sum-one-over-sum} with $a_0=\initg^2$ and $a_i=\norm{g_i}^2$,
    \begin{align*}
        \sum_{s=1}^t \eta_s^2 \norm{g_s}^2
        =
        \initeta^2 \diam^2 \sum_{s=1}^t \frac{\norm{g_s}^2}{\initg^2+\sum_{k=1}^s \norm{g_k}^2}
        \leq
        \initeta^2 \diam^2 \log \brk*{1+\frac{\sum_{s=1}^t \norm{g_s}^2}{\initg^2}}
        .
    \end{align*}
    By the inequality $\norm{u+v}^2 \leq 2 \norm{u}^2 + \norm{v}^2$, the Lipschitz assumption and the noise assumption, $\norm{g_s}^2 \leq 2 \lipball^2 + 2 \sigmaball^2$. Thus,
    \begin{align*}
        \sum_{s=1}^t \eta_s^2 \norm{g_s}^2
        &\leq
        \initeta^2 \diam^2 \log \brk*{1+\frac{2\lipball^2 T + 2 \sigmaball^2 T}{\initg^2}}
        .
        \qedhere
    \end{align*}
\end{proof}
\begin{proof}[Proof of \cref{lem:sum_tilde_eta_squared}]
    Using lemma \cref{lem:sub_gaussian}, with probability at least $1-\frac{\delta}{T}$, for $\lambda=\frac{2}{3 \sigmaball^2}$,
    \begin{align*}
        \sum_{s=1}^{t-1} \nabla f(w_s) \cdot (\nabla f(w_s) - g_s)
        &\leq
        \frac12 \sum_{s=1}^{t-1} \norm{\nabla f(w_s)}^2 + \frac32 \sigmaball^2 \log \tfrac{T}{\delta}
        .
    \end{align*}
    Thus, with probability at least $1-\frac{\delta}{T}$,
    \begin{align*}
        \sum_{s=1}^{t-1} \norm{g_s}^2
        &=
        \sum_{s=1}^{t-1} \brk{\norm{\nabla f(w_s)}^2 + \norm{g_s-\nabla f(w_s)}^2 - 2 \nabla f(w_s) \cdot (\nabla f(w_s)-g_s)}
        \\&
        \geq
        \sum_{s=1}^{t-1} \norm{g_s-\nabla f(w_s)}^2
        - 3 \sigmaball^2 \log \tfrac{T}{\delta},
    \end{align*}
    and since $\initg^2 \geq 5 \sigmaball^2 \log \tfrac{T}{\delta}$ and $\norm{g_t-\nabla f(w_t)} \leq \sigmaball$,
    \begin{align*}
        \tilde{\eta}_s
        &\leq
        \frac{\initeta \diam}{\sqrt{5 \sigmaball^2 \log \tfrac{T}{\delta} + \sum_{s=1}^{t-1} \norm{g_s}^2}}
        \leq
        \frac{\initeta \diam}{\sqrt{2 \sigmaball^2 \log \tfrac{T}{\delta} + \sum_{s=1}^{t-1} \norm{g_s-\nabla f(w_s)}^2}}
        \\&
        \leq
        \frac{\initeta \diam}{\sqrt{\sigmaball^2 + \sum_{s=1}^{t} \norm{g_s-\nabla f(w_s)}^2}}
        .
    \end{align*}
    Hence, under a union bound, with probability at least $1-\delta$, for any $t \in [T]$,
    \begin{align*}
        \tilde{\eta}_s
        \leq
        \frac{\initeta \diam}{\sqrt{\sigmaball^2 + \sum_{s=1}^{t} \norm{g_s-\nabla f(w_s)}^2}}
    \end{align*}
    and using \cref{lem:sum-one-over-sum} with $a_0=\sigmaball^2$ and $a_i = \norm{g_i-\nabla f(w_i)}^2$,
    \begin{align*}
        \sum_{s=1}^t \tilde{\eta}_s^2 \norm{g_s-\nabla f(w_s)}^2
        &\leq
        \initeta^2 \diam^2 \sum_{s=1}^t \frac{\norm{g_s-\nabla f(w_s)}^2}{\sigmaball^2 + \sum_{s=1}^t \norm{g_s-\nabla f(w_s)}^2}
        \\&
        \leq
        \initeta^2 \diam^2 \log \brk*{\frac{\sigmaball^2 + \sum_{s=1}^t \norm{g_s-\nabla f(w_s)}^2}{\sigmaball^2}}
        \leq
        \initeta^2 \diam^2 \log (1+T)
        .
        \qedhere
    \end{align*}
\end{proof}

\subsection{Proof of \texorpdfstring{\cref{lem:padasgd-convergence-martingale}}{Lemma 7}}
\begin{proof}[\unskip\nopunct]
    Let
    \begin{align*}
        X_s = (\nabla f(w_s)-g_s) \cdot (w_s-w^\star)
        ,\qquad
        \hat{X}_s = 0
        ,\qquad\text{and}\qquad
        c = \sigmaball(\diam+\norm{w_1-w^\star}).
    \end{align*}
    Note that by Cauchy-Schwarz inequality and the triangle inequality,
    \begin{align*}
        (\nabla f(w_s)-g_s) \cdot (w_s-w^\star)
        &\leq
        \norm{\nabla f(w_s)-g_s} \norm{w_s-w^\star}
        \leq
        \norm{\nabla f(w_s)-g_s} \brk{\norm{w_s-w_1}+\norm{w_1-w^\star}}
        \\&
        \leq
        \sigmaball \brk{\diam+\norm{w_1-w^\star}}
        =
        c
        ,
    \end{align*}
    where we used the assumption that $\diam \leq 8 \norm{w_1-w^\star}$ to bound $\norm{\nabla f(w_s)-g_s}$.
    Applying \cref{lem:emp-bernstein}, with probability at least $1-\delta$,
    \begin{align*}
        \abs*{\sum_{t=1}^T (\nabla f(w_t)-g_t) \cdot (w_s-w^\star)} \leq 4 (\diam+\norm{w_1-w^\star}) \sqrt{\theta_{T,\delta} \sum_{t=1}^T \norm{\nabla f(w_t)-g_t}^2 + \sigmaball^2 \theta_{T,\delta}^2}
        .
    \end{align*}
    Applying \cref{lem:sub_gaussian} with $Z_t = \nabla f(w_t) \cdot (\nabla f(w_t)-g_t)$ and $\lambda=\frac{2}{3 \sigmaball^2}$, with probability at least $1-\delta$
    \begin{align*}
        \sum_{t=1}^T \nabla f(w_t) \cdot (\nabla f(w_t)-g_t)
        &\leq
        \frac{3 \sigmaball^2}{4} \lambda \sum_{t=1}^T \norm{\nabla f(w_t)} + \frac{\log \tfrac{1}{\delta}}{\lambda}
        =
        \frac{1}{2} \sum_{t=1}^T \norm{\nabla f(w_t)} + \frac{3 \sigmaball^2 \log \tfrac{1}{\delta}}{2}
        .
    \end{align*}
    Hence,
    \begin{align*}
        \sum_{t=1}^T \norm{\nabla f(w_t)-g_t}^2
        &=
        \sum_{t=1}^T \brk*{\norm{g_t}^2+2 \nabla f(w_t) \cdot (\nabla f(w_t)-g_t)-\norm{\nabla f(w_t)}^2}
        \\
        &\leq
        3 \sigmaball^2 \log \tfrac{1}{\delta} + \sum_{t=1}^T \norm{g_t}^2
        .
    \end{align*}
    Returning to the previous inequality, with probability at least $1-2\delta$ (union bound),
    \begin{align*}
        \abs*{\sum_{t=1}^T (\nabla f(w_t)-g_t) \cdot (w_s-w^\star)}
        &\leq
        4 (\diam+\norm{w_1-w^\star}) \sqrt{\theta_{T,\delta} \sum_{t=1}^T \norm{g_t}^2 + \theta_{T,\delta} \brk*{3 \log \tfrac{1}{\delta} + \theta_{T,\delta}} \sigmaball^2}
        \\
        &\leq
        4 (\diam+\norm{w_1-w^\star}) \sqrt{\theta_{T,\delta} \sum_{t=1}^T \norm{g_t}^2 + 4 \theta_{T,\delta}^2 \sigmaball^2}
    \end{align*}
    where the last inequality follows by $\theta_{T,\delta} \geq \log \tfrac{1}{\delta}$. Thus, $\Pr(\event_2) \geq 1-2\delta$.
\end{proof}

\section{Assumption for Noise Upper Bound}\label{sec:estimation}
In the convex non-smooth setting, \cref{thm:lower-bound-new} implies that a stochastic convex optimization method cannot be parameter-free with respect to all three parameters, diameter, Lipschitz bound and noise bound, while being competitive with tuned SGD. Next, we observe that while such a term is unavoidable without further assumptions, if a success probability of $1-O(\delta)$ is desired, the following reasonable assumption is sufficient.

\noindent\textbf{Assumption:} for some $\delta \in (0,1)$ and $c = 4 \sqrt{1+3 \log \tfrac{1}{\delta}}$,
\begin{align*}
    f(w_1)-f^\star > \frac{c \sigmastar \norm{w_1-w^\star}}{\sqrt{T}}.
\end{align*}
In other words, $w_1$ is a ``bad enough'' approximate minimizer.

Let us first explain why this is a reasonable assumption. Consider the convergence rate of tuned SGD, 
which ensures with probability at least $1-\delta$ that
\begin{align*}
    f(\wout)-f^\star \leq
        \frac{C \norm{w_1-w^\star} (\lipstar+\sigmastar \sqrt{\log\tfrac{1}{\delta}})}{\sqrt{T}}
\end{align*}
for some constant $C > 0$.
In that case, the assumption is weaker then assuming that the above convergence bound is sufficient to ensure that with probability at least $1-\delta$, SGD produce a solution better than $w_1$ by some constant factor. If we aim to compete with tuned SGD, it is reasonable to focus on the regime where SGD produce better solutions than the initialization.

We move to analyze what the assumption above yields.
Using the assumption and the convexity of the objective,
\begin{align*}
    \frac{c \sigmastar \norm{w_1-w^\star}}{\sqrt{T}}
    &<
    f(w_1)-f^\star
    \leq
    \nabla f(w_1) \cdot (w_1-w^\star)
    \leq
    \norm{\nabla f(w_1)} \norm{w_1-w^\star}
    .
\end{align*}
Thus,
\begin{align*}
    \sigmastar
    &<
    \frac{\norm{\nabla f(w_1)} \sqrt{T}}{c}
    .
\end{align*}
So upper bounding $\norm{\nabla f(w_1)}$ produce an upper bound of $\sigmastar$.
The simplest approach to estimate the gradient at $w_1 \in \reals^d$ is to average $n$ stochastic samples. Let $Z=\frac{1}{n} \sum_{i=1}^n \goracle_i(w_1)$ where $\goracle_i(w_1)$ for $i \in [n]$ are independent stochastic gradients. Thus, $\E[Z]=\nabla f(w_1)$.
By \cref{lem:mds-sub-gaussian-bound},
\begin{align*}
    \Pr\brk*{\norm{Z-\nabla f(w_1)} \geq \ifrac{\sqrt{2}(1+\lambda) \sigma}{\sqrt{n}}}
    &\leq
    \exp\brk{\ifrac{-\lambda^2}{3}}
    .
\end{align*}
Let $\lambda=\sqrt{3 \log \tfrac{1}{\delta}}$. Hence, with probability at least $1-\delta$,
\begin{align*}
    \norm{Z-\nabla f(w_1)}
    &<
    \sqrt{2}\ifrac{\brk{1+\sqrt{3 \log \tfrac{1}{\delta}} \sigma}}{\sqrt{n}}
    <
    \sqrt{2}\ifrac{\brk{1+\sqrt{3 \log \tfrac{1}{\delta}}} \norm{\nabla f(w_1)} \sqrt{T}}{c \sqrt{n}}
    .
\end{align*}
Setting $n=\ifrac{8 \brk{1+3\log \tfrac{1}{\delta}} T}{c^2}=\tfrac{T}{2}$, with probability at least $1-\delta$,
\begin{align*}
    \frac{\norm{Z}}{\norm{\nabla f(w_1)}} \in [\tfrac12,\tfrac32]
    .
\end{align*}
Hence, using $\tfrac{T}{2}$ steps we can estimate $\norm{\nabla f(w_1)}$ and obtain a bound
\begin{align*}
    \sigmastar < \frac{\norm{\nabla f(w_1)} \sqrt{T}}{c}
    \leq
    \frac{2 \norm{Z} \sqrt{T}}{4 \sqrt{1+3 \log \tfrac{1}{\delta}}}
\end{align*}
which holds with probability at least $1-\delta$, and such bound is tight enough to make \cref{thm:lower-bound-new} inapplicable.
Further note that having a bound $\sigmastar \leq \sigmamax = O(\lipstar\sqrt{T})$, which we obtain with high probability using this technique, is sufficient to reduce the unavoidable lower order term presented at \cref{thm:convex-ball} to a term proportional to the (almost) optimal rate of tuned SGD.

\end{document}